\documentclass[lettersize,journal]{IEEEtran}
\usepackage{amsmath,amsfonts}
\usepackage{algorithmicx}
\usepackage{algorithm}
\usepackage{array}
\usepackage{textcomp}
\usepackage{stfloats}
\usepackage{url}
\usepackage{verbatim}
\usepackage{graphicx}
\usepackage{cite}

\usepackage{algpseudocode}
\usepackage{amsthm}
\usepackage{subcaption}
\usepackage{hyperref}
\usepackage[numbers]{natbib}
\usepackage{verbatim}
\usepackage{enumitem}
\usepackage{booktabs}
\usepackage{multirow}
\usepackage{tabularx}
\usepackage{siunitx}
\usepackage{caption}
\usepackage{cleveref}
\crefformat{subsection}{Section~#2#1#3}
\Crefformat{subsection}{Section~#2#1#3}

 \usepackage{threeparttable}

\makeatother

\usepackage{xcolor}

\newcommand{\rfig}[1]{\autoref{fig:#1}}
\newcommand{\ralg}[1]{\autoref{alg:#1}}
\newcommand{\rthm}[1]{\autoref{thm:#1}}

\newcommand{\reqn}[1]{\autoref{eqn:#1}}
\newcommand{\rtbl}[1]{\autoref{tbl:#1}}
\newcommand{\rsec}[1]{\cref{subsec:#1}}
\newcommand{\Rsec}[1]{\Cref{subsec:#1}}

\algnewcommand\Null{\textsc{null }}
\algnewcommand\algorithmicinput{\textbf{Input:}}
\algnewcommand\Input{\item[\algorithmicinput]}
\algnewcommand\algorithmicoutput{\textbf{Output:}}
\algnewcommand\Output{\item[\algorithmicoutput]}
\algnewcommand\algorithmicbreak{\textbf{break}}
\algnewcommand\Break{\algorithmicbreak}
\algnewcommand\algorithmiccontinue{\textbf{continue}}
\algnewcommand\Continue{\algorithmiccontinue}
\algnewcommand{\LeftCom}[1]{\State $\triangleright$ #1}

\algnewcommand{\OR}{\textbf{or}~}
\algnewcommand{\AND}{\textbf{and}~}

\newtheorem{thm}{Theorem}
\newtheorem{lem}{Definition}

\colorlet{shadecolor}{black!15}

\theoremstyle{definition}

\def\thmautorefname~#1\null{Theorem~#1~\null}
\def\lemautorefname~#1\null{Define~#1~\null}
\def\algorithmautorefname~#1\null{Algorithm~#1~\null}



\begin{document}

%\title{Solving the Partitioning Min-Max Weighted Matching Problem by Fast Iterative Match-Partition Hybrid Genetic Algorithm with Elite Strategy}

%\title{A Fast Iterative Match-Partition Hybrid Genetic Algorithm for the Partitioning Min-Max Weighted Matching Problem}
\title{FIMP-HGA: A Novel Approach to Addressing the Partitioning Min-Max Weighted Matching Problem}

\author{Yuxuan Wang, %d202381468@hust.edu.cn
Jiongzhi Zheng, % jzzheng@hust.edu.cn
Jinyao Xie, %735289238@qq.com , xiejinyao2@huawei.com, 
Kun He$^{*}$, ~\IEEEmembership{Senior~Member,~IEEE}
        % <-this % stops a space
%\thanks{This paper was produced by the IEEE Publication Technology Group. They are in Piscataway, NJ.}% <-this % stops a space
\thanks{This work was supported by the National Natural Science Foundation (U22B2017).}
% \thanks{Manuscript received Feb. 27, 2024; revised ?? ??, 202?. (Corresponding author: Kun He. Email: brooklet60@hust.edu.cn.)}
}

% \address{School of Computer Science and Technology, Huazhong University of Science and Technology, Wuhan 430074, China}

% The paper headers

\begin{comment}
\markboth{IEEE Transactions on Evolutionary Computation,~Vol.~??, No.~?, August~202?}%
{Wang \MakeLowercase{\textit{et al.}}: FIMP-HGA: A Novel Approach to Addressing the Partitioning Min-Max Weighted Matching Problem}
\end{comment}

% \IEEEpubid{0000--0000/00\$00.00~\copyright~2021 IEEE}
% Remember, if you use this you must call \IEEEpubidadjcol in the second
% column for its text to clear the IEEEpubid mark.

\maketitle

\begin{abstract}
The Partitioning Min-Max Weighted Matching (PMMWM) problem, being a practical NP-hard problem, integrates the task of partitioning the vertices of a bipartite graph into disjoint sets of limited size with the classical Maximum-Weight Perfect Matching (MPWM) problem. 
Initially introduced in 2015, the state-of-the-art method for addressing PMMWM is the MP$_{\text{LS}}$. 
In this paper, we present a novel approach, the Fast Iterative Match-Partition Hybrid Genetic Algorithm (FIMP-HGA), for addressing PMMWM. Similar to MP$_{\text{LS}}$, FIMP-HGA divides the solving into match and partition stages, iteratively refining the solution. 
In the match stage, we propose the KM-M algorithm, which reduces matching complexity through incremental adjustments, significantly enhancing runtime efficiency. 
For the partition stage, we introduce a Hybrid Genetic Algorithm (HGA) incorporating an elite strategy and design a Greedy Partition Crossover (GPX) operator alongside a Multilevel Local Search (MLS) to optimize individuals in the population. Population initialization employs various methods, including the multi-way Karmarkar-Karp (KK) algorithm, ensuring both quality and diversity. 
At each iteration, the bipartite graph is adjusted based on the current solution, aiming for continuous improvement. 
To conduct comprehensive experiments, we develop a new instance generation method compatible with existing approaches, resulting in four benchmark groups. %Experimentation
Extensive experiments evaluate various algorithm modules, accurately assessing each module's impact on improvement. Evaluation results on our benchmarks demonstrate that the proposed FIMP-HGA significantly enhances solution quality compared to MP$_{\text{LS}}$, meanwhile reducing runtime by 3 to 20 times.
\end{abstract}


% old struct
\begin{comment}
\section{Introduction}
\input{01-Intro}

\section{Problem Formulation}
\input{02-Prob}

\section{Related Works}
\input{03-RW}

% 改进框架Fast MP
\section{Fast Matching-Partition Based On Incremental Pathfinding}
\input{04-FMP}

\section{More Effective Heuristic Strategies for Partition Phase}
\input{05-HEE}

\section{Experiment Result}
\input{06-Exp}

\section{Conclusion}
The conclusion goes here.

\end{comment}

% new struct
\section{Introduction}
\label{subsec:intro}
\IEEEPARstart{I}{n} this paper, we consider %a typical NP-hard problem, the Partitioning Min-Max Weighted Matching (PMMWM) problem~\cite{a1}. 
the Partitioning Min-Max Weighted Matching (PMMWM) problem that is NP-hard in the strong sense~\cite{a1}. 
The PMMWM revolves around a weighted bipartite graph $G(U, V, E)$ comprising two disjoint vertex sets $U$, $V$ and an edge set $E = \{e_{uv} | u \in U, v \in V\}$, where the weight of each edge $e_{uv}$ is denoted by $w(e_{uv})$. Set $U$ is required to be partitioned into $m$ disjoint partitions, each partition is restricted to a maximum of $\bar{u}$ vertices.
%Given a maximum matching on $G$, the weight of a partition is the cumulative sum of edge weights for edges matching the vertices within the partition. 
For a maximum matching on $G$, the weight of a partition is determined by the total sum of the weights of edges that match vertices within that partition.
The objective of the problem is to find a matching and a partition that together minimize the weight of the heaviest partition. %The PMMWM is an NP-hard problem in the strong sense~\cite{a1}.

The PMMWM problem serves as an extension of the Min-Max Weighted Matching (MMWM) problem \citep{a2}, which originates from the container transshipment operations in rail-road terminals. While MMWM shares the same objective function as PMMWM, it operates with predetermined vertex partitions and focuses solely on determining the matching result.
%MMWM has been proven to be 
Though being NP-hard in the strong sense, %and effective heuristics for MMWM 
MMWM has found effective heuristics that can find matching results %whose average gap to the optimal solution is less than 1\% have been developed~\cite{a2}. 
with an average optimality gap of less than 1\% \cite{a2}.
However, in some practical applications, the partition is not predetermined but can be adjusted to further optimize the objective. 
\citet{a1} introduce the PMMWM problem with its representative application at small to medium-sized seaports containing long-term and temporary storage areas.
% It involves moving containers from temporary to long-term storage, where each move incurs a cost represented by edge weights. The challenge is to manage these relocations efficiently within the constraints of available telescopic stackers and their capacity limits. 
The challenge in this context lies in container relocation, where each move incurs a specific cost denoted by edge weights. 
%its representative application at small to medium-sized seaports containing long-term and temporary storage areas. The vertex set $U$ denotes the empty slots in the long-term storage area, the vertex set $V$ denotes the containers in the temporary storage area, and the edge weight $w(e_{uv})$ denotes the cost required to move the container $v \in V$ to the empty slot $u \in U$. The parameter $m$ indicates the number of available telescopic stackers and $\bar{u}$ indicates the upper limit of the number of containers that can be handled by one telescopic stacker. Suppose the cost of a telescopic stacker equals to the total cost of moving its containers, the goal of the problem is to minimize the maximum cost among the telescopic stackers.
Additionally, %we find another practical application of the PMMWM relates to dealing with optimizing task allocation in a workshop. 
we find another practical application of PMMWM that lies in optimizing task allocation in workshops. 
Each task should be performed on a machine by a worker, with each worker undertaking only one task at a time, and each machine can handle at most one task simultaneously. The goal is to determine a matching between workers and tasks, and partition workers among machines, to minimize the longest %spending time among the machines.
running time of the machines, i.e., minimize the makespan.

% We find another practical application of the PMMWM relates to task allocation. Suppose there are $n$ workers, $n$ tasks, and $m$ machines in a workshop. Each task should be performed on a machine by a worker, each worker performs only one task, and each machine can take at most one task at a time. The vertex set $U$ denotes the workers, $V$ denotes the tasks, and the edge weight $w(e_{uv})$ denotes the spending time in performing task $v \in V$ by worker $u \in U$. The parameter $\bar{u}$ indicates the maximum number of tasks that can be performed by each machine per day (working period). The goal of the problem is to determine the matching between workers and tasks, and the partition of workers to machines, to minimize the longest spending time among the machines.

\begin{figure}[t]
\centering
\begin{subfigure}{0.24\textwidth}
  \includegraphics[width=\linewidth]{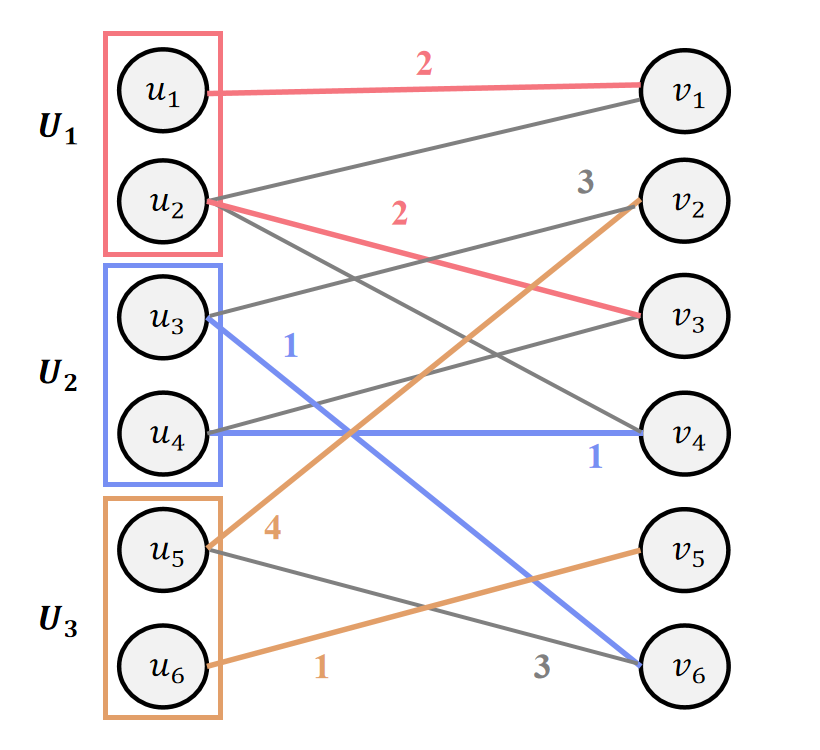}
  \caption{}
  \label{fig:example1}
\end{subfigure}\hfil
\begin{subfigure}{0.24\textwidth}
  \includegraphics[width=\linewidth]{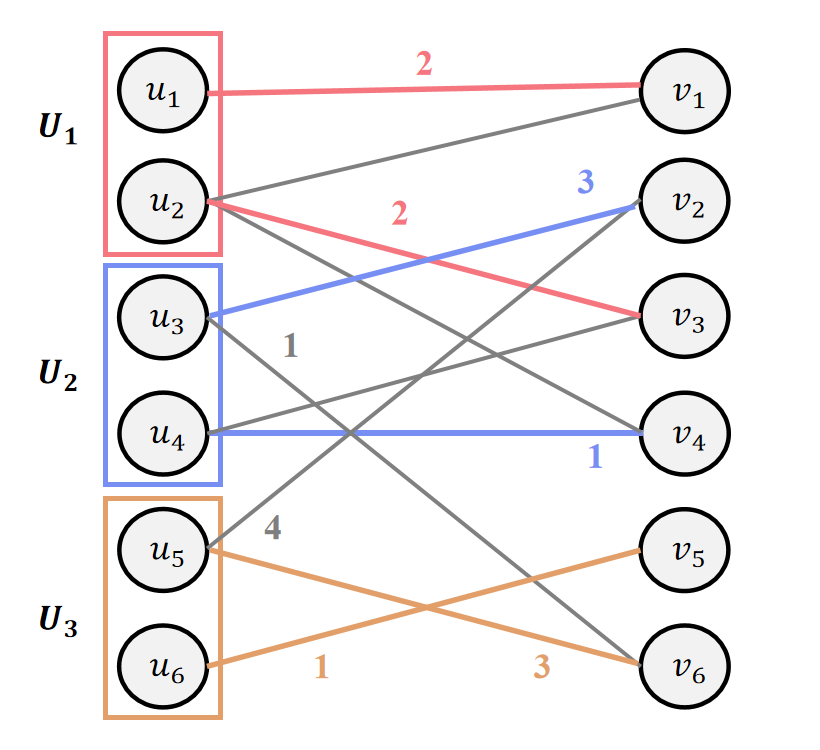}
  \caption{}
  \label{fig:example2}
\end{subfigure}
\caption{(a) A feasible PMMWM solution. (b) Improve the solution quality from (a) by modifying the matching scheme.}
\label{fig:pmmwm}
\end{figure}

To facilitate a deeper understanding of PMMWM and its components, we present an instance in \rfig{pmmwm}, where \rfig{example1} illustrates a feasible solution to the instance. Colored edges in the figure signify the maximum matching edges with associated weights marked alongside. The left part of the bipartite graph, $U=\{u_1,u_2,...,u_6\}$, is divided into three partitions, $U_1$, $U_2$, and $U_3$ distinguished by different colors. The weights of the three partitions are 4, 2, and 5, respectively, resulting in an objective value of 5. %We consider two different ways to improve the solution. %optimize the objective function.
This solution quality can be improved through two distinct approaches. 
The first method involves adjusting the partition strategy. By relocating $u_6$ to $U_2$, the weights of the partitions transform to 4, 3, and 4, consequently yielding an objective value of 4. 
The second approach is to modify the matching strategy. For $u_3$ and $u_5$, replacing the current matching edges $e_{u_3 v_6}, e_{u_5 v_2}$ with edges $e_{u_3 v_2},e_{u_5 v_6}$, respectively, leads to the partition weights of 4, 4, 4, as shown in \rfig{example2}. Although this adjustment increases the total weight of matching edges, it effectively reduces the final objective value to 4.

From the above example, we can observe that the PMMWM problem can be decomposed into two stages: match and partition. Adjustments made to either of these stages hold the potential to optimize the solution quality. 
%, which can be solved either exactly or heuristically. 
The match stage requires solving the minimum weight perfect matching problem, which can be reformulated into the classic Maximum-Weight Perfect Matching (MWPM) problem \cite{MWPM, AP} by inverting the weight values $w'(e_{uv}) = -w(e_{uv})$. There are exact polynomial-time algorithm, such as the KM algorithm \cite{KM} (also known as the Hungarian method), can find the optimal solution to the MWPM problem. The partition stage requires distributing vertices into capacity-constrained partitions, aiming to minimize the weight of the heaviest partition.
This problem has been addressed in various application scenarios, such as the Multiprocessor Scheduling problem \cite{RP} and the Parallel Machine Scheduling problem \cite{PMS}. Additionally, the Multi-way Number Partition (MNP) problem \cite{NPAll-1, NPAll-2,NPAll-3} is another similar problem aiming to minimize the range of partition weights. 
% 解决该类问题的方法大致有两类：optimal algorithm, approximate algorithm and heuristic algorithm. 最优算法如SNP,MOF,BSBCP,CIW虽然能得到最优解，但只能应用在n较小的情况，在n=60下就需要花费1分钟左右的运行时间。approximate algorithm,如GA和KK能较快的得到一组保证质量的解，但往往离最优解有一定距离。Heuristic algorithm介于这两者之间，如DIMM-SS,ICSA能在一定时间内得到不错的解。

Algorithms for solving these partition problems typically fall into three main categories: 
exact algorithms, approximate algorithms, and heuristic algorithms. Exact algorithms \cite{SNP, MOF,BSBCP,CIW} can achieve optimal solutions but are limited to small instances because of their high computational complexity. 
% while capable of obtaining optimal solutions, are only applicable to small instances due to the computational complexity. 
Approximate algorithms \cite{GA1,KK}, on the other hand, can quickly produce solutions with guaranteed quality but may exhibit a notable gap from the optimal solution. As a good trade-off between the above two categories, heuristic algorithms \cite{DIMM,ICSA} offer satisfactory solutions within reasonable runtime. 
% PMMWM问题的算例超过了精确算法能解决的规模，而近似算法的优度又不够，因此一般采用启发式算法求解partition stage。
%The scale of the PMMWM problem instances exceeds the capacity of exact algorithms, and the quality of solutions provided by approximate algorithms is not sufficient. Therefore, heuristic algorithms are generally used to solve the partition stage.
Considering the scale of PMMWM problem instances, exact algorithms may struggle to cope,  while the solution quality from approximate algorithms might not meet requirements. Consequently, heuristic algorithms are commonly employed to tackle the partition stage.

\citet{a1} introduce the MP$_{\text{LS}}$ algorithm for PMMWM, which employs the KM algorithm to the match stage and a local search method to the partition stage and iterates these two stages to optimize the solution quality. Once starting a new iteration, MP$_{\text{LS}}$ adjusts the bipartite graph slightly and calls the KM algorithm again to find a maximum matching. 
%respectively to these two stages, and conducting iterative optimization. 
Instead, \citet{7376822} employ an integrated approach with tabu search  \citep{DBLP:journals/informs/Glover89, DBLP:journals/informs/Glover90} and genetic algorithm \citep{DBLP:books/daglib/0019083} to directly solve the problem. %They propose a heuristic based on tabu search~\citep{DBLP:journals/informs/Glover89,DBLP:journals/informs/Glover90} and a genetic algorithm~\citep{DBLP:books/daglib/0019083}. 
%But these two algorithms are not as effective as the MP$_{\text{LS}}$ algorithm. Therefore, MP$_{\text{LS}}$ is still the state-of-the-art algorithm for the PMMWM. 
Results indicate that the decomposition approach of MP$_{\text{LS}}$ can explore the solution space more efficiently than the integrated approach~\cite{7376822}. However, MP$_{\text{LS}}$ still has room for improvement in both search efficiency during the match stage and search capability during the partition stage.

%Based on previous experimental~\cite{7376822}, the decomposition approaches can explore the solution space more efficiently than the integrated approaches. Therefore, 
In this paper, we adopt the decomposition Match-Partition framework and propose a Fast Iterative Match-Partition Hybrid Genetic Algorithm (FIMP-HGA) with elite strategy to solve the PMMWM problem and address the limitation of MP$_{\text{LS}}$. 
%FIMP-HGA is an iterative heuristic algorithm, and each iteration consists of three stages.  
FIMP-HGA operates as an iterative heuristic algorithm, with each iteration comprising three distinct stages.
In the first stage (match stage), we employ an exact algorithm to solve the MWPM problem. While the KM algorithm utilized in MP$_{\text{LS}}$ guarantees optimal solutions, 
%its running efficiency is unsatisfactory. 
its efficiency in runtime leaves much to be desired.
% 由于每轮迭代中只对图作细微的调整，从第二轮迭代开始可以复用之前大量的匹配信息，从而减少冗余操作。
Recognizing that only minor adjustments are made to the graph in each iteration, %a large amount of matching information from previous iterations can be reused to reduce redundant operations. 
we capitalize on the reuse of substantial matching information from preceding iterations to mitigate redundant operations. 
% Guided by this analysis, 
Motivated by this insight, we propose the KM-M algorithm, which accelerates the complexity of the match stage from $O(n^3)$ to $O(n^2)$ while preserving the optimality of the matching.
% 演化算法相较于用于MP_LS中的局部搜索能提供更好的全局搜索能力，并且提供更好的鲁棒性和灵活性.genetic algorithm 作为演化算法的一种，通过

Consequently, we propose a Hybrid Genetic Algorithm (HGA) with elite strategy for the second stage (partition stage). Unlike the local search utilized in MP$_{\text{LS}}$, the genetic algorithm (GA) \cite{Genetic1,Genetic2,Genetic3}, as a type of evolutionary algorithms~\cite{Evolution1,Evolution12,Evolution13}, offers better global search capabilities, along with enhanced robustness and flexibility. HGA hybrids the genetic algorithm and local search method, performing both global and local search schemes, enabling the algorithm to explore the solution space widely and deeply.
Specifically, HGA leverages approximate algorithms, greedy algorithm~\cite{GA1}, and the multi-way Karmarkar-Karp (KK) algorithm \cite{KK}, to generate the initial population, and generates new populations by our proposed Greedy Partition Crossover (GPX) operator. Each solution obtained through initialization and crossover will be further improved by our proposed Multilevel Local Search (MLS) algorithm. To accelerate the evolution of the population, we introduce the elite strategy \cite{GeneticElite1,GeneticElite2,GeneticElite3}, ensuring that in each crossover, the optimal individual always survives to the subsequent generation.

In the final stage, FIMP-HGA bans an edge in %$G'$ 
the graph for a certain number of iterations to adjust the bipartite matching solution in subsequent iterations, which can help the algorithm explore a broader solution space.
% 但是如果这轮迭代得到的解和最优解的差距到达一定值时，将会把当前所有禁用的边以一定概率马上解禁。当边解禁的同时也需要用KM-M算法对相关节点重新规划匹配方案。
However, if the difference between the solution obtained in the current iteration and the optimal solution reaches a certain threshold, all currently banned edges will be released with a certain probability. This approach is named the edge recovery strategy. When an edge is released, it is also necessary to use the KM-M algorithm to recalculate the matching scheme for the related vertices.

Furthermore, we categorize the instances based on the consistency and density of the bipartite graph, thereby generating new benchmarks compatible with prior studies. We then conduct extensive experiments using FIMP-HGA and analyze the characteristics of the PMMWM problem based on these new benchmarks.

The remainder of the paper is organized as follows.
\Rsec{prob} introduces a detailed formulation of the PMMWM problem. Subsequently, \Rsec{p} presents the preliminaries, including an overview of the KM algorithm and the KK algorithm. \Rsec{alg} presents details of the FIMP-HGA. The process of instance generation and the corresponding experimental results are detailed in \Rsec{exp}. Finally, \Rsec{con} summarizes the key findings and conclusions drawn in this study.

\section{Problem Formulation}
\label{subsec:prob}
Let $G(U,V,E)$ be a weighted bipartite graph, where $U$ and $V$ are two disjoint vertex sets, and $E = \{e_{uv}|u \in U, v \in V\}$ denotes the set of edges connecting vertices in $U$ and $V$. 
Assume $|U| = n_1$, $|V| = n_2$, and $n_1 \leq n_2$. Each edge $e_{uv} \in E$ is associated with a weight $w(e_{uv})\in\mathbb{Q}^{+}_{0}$. Define a matching in $G$ as the set $M\subseteq E$ of nonadjacent edges, and the maximum matching~\citep{DBLP:journals/computing/DerigsZ78} as the matching with the largest $|M|$ among all matchings in $G$. 

%Let $G(U,V,E)$ be a weighted bipartite graph. The vertices of $U$ and $V$ are indexed by $i=1,...,n_{1}$ and $j=1,...,n_{2}$, respectively, assuming $n_{1}\le n_{2}$. The weight $c(e)=c_{uv}\in\mathbb{Q}^{+}_{0}$ is associated with each edge of $G$. Define a matching as the set $M\subseteq E$ of pairwise nonadjacent edges and the maximum matching\cite{DBLP:journals/computing/DerigsZ78} as the matching with the largest $|M|$ among all matchings on $G$. 

Assume that for any given weighted bipartite graph, there exists at least one maximum matching $\Pi$ such that $|\Pi|=n_{1}$. Define a partition $\mathcal{P}$ of $U$ that divides $U$ into $m$ disjoint partitions $U_1,U_2,...,U_m$, with at most $\bar{u}$ vertices in each partition. 
For a given matching $\Pi$, define the weight of partition $U_i$ as $W(U_i) = \sum_{u \in U_i, e_{u,v} \in \Pi} w(e_{uv})$. 
The goal of the PMMWM is to find a partition $\mathcal{P}$ of $U$ and a maximum matching $\Pi$ of $G$ that minimizes the objective function $f(\Pi, \mathcal{P}) := \max_{k\in\{1, ...,m\}}W(U_k)$.

%Assume that for any given weighted bipartite graph, there exists a maximum matching $\Pi$ such that $|\Pi|=n_{1}$, and assume that $U$ is assigned into $m$ disjoint partitions $U_1,U_2,...,U_m$, with at most $\bar{u}$ vertices in each partition, the value of the maximum matching $\Pi$ is defined as $w(\Pi):=\max_{k\in\{1, ...,m\}}\{\sum_{u\in U_k,(u,v)\in\Pi}c_{uv}\}$.

%Based on the above assumptions, the PMMWM can be defined as follows: assigning the vertex set $U$ into $m$ possibly empty disjoint partitions $U_1,U_2,...,U_m$ with at most $\bar{u}$ vertices in each partition, and find a maximum matching $\Pi$ on $G$ such that the value of $\Pi$ is the smallest among all possible assignment of $U$.

Let $z_{uvk}$ be a binary variable such that $z_{uvk} = 1$ if $u \in U_k, e_{uv} \in \Pi$, and $z_{uvk} = 0$ otherwise. 
%The mathematical model of the PMMWM is as follows.
The PMMWM problem can be formalized as follows. 

\begin{equation*}
\begin{aligned}
    \label{eq:pmmwm}
    & \min_{\mathbf{z}} \max_{k\in\{1,…,m\}} \{\sum_{u\in U}\sum_{v\in V}w(e_{uv})z_{uvk}\} \\
    %\label{eq:pmmwm_1}
\text{s.t.} \qquad & (1) \qquad\sum_{k=1}^{m}\sum_{v\in V}z_{uvk}=1\quad\forall u\in U,\\
    &(2) \qquad \sum_{k=1}^{m}\sum_{u\in U}z_{uvk}\le1\quad\forall v \in V,\\
    &(3) \qquad \sum_{k=1}^{m}z_{uvk}=1\quad\forall e_{uv} \in E,\\
    &(4) \qquad \sum_{u \in U}\sum_{v\in V}z_{uvk}\le\bar{u}\quad\forall k\in\{1,…,m\}.
\end{aligned}
\end{equation*}

\iffalse
\begin{equation}
    \label{eq:pmmwm_2}
    (2) \qquad \sum_{k=1}^{m}\sum_{u\in U}z_{uvk}\le1\quad\forall v \in V,
\end{equation}
    
\begin{equation}
    \label{eq:pmmwm_3}
    (3) \qquad \sum_{k=1}^{m}z_{uvk}=1\quad \forall u\in U, v \in V,
\end{equation}
    
\begin{equation}
    \label{eq:pmmwm_4}
    (4) \qquad \sum_{u \in U}\sum_{v\in V}z_{uvk}\le\bar{u}\quad\forall k\in\{1,…,m\}.
\end{equation}
\fi

% 目标函数最小化所最大匹配匹配和合法分区方案中最大分区的权重
%The objective function minimizes the weight of maximum matching weight among all possible assignments and maximum matchings. \YX{logic prob ? }
The objective function aims to find the minimum weight of the maximum partition weight among all possible maximum matchings and valid partitioning schemes.
Constraints (1) and (2) are the well-known constraints on the maximum matching. Constraint (3) forces each vertex $u\in U$ to belong to exactly one partition $U_k,k\in \{1,...,m\}$. Constraint (4) restricts that the size of the partition does not exceed $\bar{u}$. %\eqref{eq:pmmwm_5} and \eqref{eq:pmmwm_6} are the domains of the variables.

\section{Preliminaries}
\label{subsec:p}
% 这一章首先介绍PMMWM问题的相关研究，然后分别介绍$FMP_{HEE}$中会使用或涉及的算法，有KM算法和multi-way Karmarkar-karp algorithm。
%This chapter first introduces research related to the PMMWM problem. It then separately discusses the algorithms used or involved in FMP\_{HEE}, including the KM and the multi-way Karmarkar-Karp algorithm.

% 在FMP\_HEE中，在匹配阶段我们根据经典的KM算法提出了KM-M algorithm。在partition阶段通过Multi-way Karmarkar-Karp Algorithm生成更高质量的初始解。接下来分别介绍KM和Multi-way Karmarkar-Karp Algorithm。

%In FIMP-HGA, during the matching phase, we introduce the KM-M algorithm, which is based on the classical KM algorithm~\cite{KM}. 
In the match stage, our proposed the KM-M algorithm is based on the classical KM algorithm \cite{KM}.
In the partition stage, high-quality initial solutions are generated through the multi-way Karmarkar-Karp (KK) algorithm \cite{KK}. In the following, we will describe the KM and KK algorithms that are closely related to our FIMP-HGA.

\subsection{KM Algorithm}
The KM algorithm is used to solve the Maximum-Weight Perfect Matching (MWPM) problem~\cite{MWPM}, which aims to find a perfect matching with the maximum total weight of the matched edges. 
% Since PMMWM and MWPM are minimization and maximization problems, respectively, 
% In the KM process, we negative the original edge weight to be $w'(e_{uv}) = -w(e_{uv})$.
%Define the weight during the matching stage as $w'(e_{uv}) = -w(e_{uv})$.
We first introduce some definitions and then describe the procedure of the KM algorithm.

\begin{lem}[Alternating Path]
    Given a bipartite graph $G = (U,V,E)$ with some edges matched. An alternating path on $G$ is a path that starts from an unmatched vertex and then traverses an unmatched edge and a matched edge alternatively.
\end{lem}
	
\begin{lem}[Augmenting Path]
    Given a bipartite graph $G = (U,V,E)$ with some edges matched. An augmenting path in $G$ is a special alternating path on $G$ whose starting and ending vertices are both unmatched.
\end{lem}
	
\begin{lem}[Feasible Label]
    A label of a bipartite graph $G = (U,V,E)$ can be represented by assigning a value $ex_i$ to each vertex $i \in U \cup V$. A feasible label is a label that satisfies $ex_u+ex_v \ge w'(e_{uv})$ for each edge $e_{uv} \in E$.
\end{lem}
	
\begin{lem}[Equivalence Subgraph]
    An equivalence subgraph is a spanning subgraph of the original graph (a spanning subgraph contains all vertices of the original graph, but not all edges) that only contains edges satisfying $ex_u+ex_v=w'(e_{uv})$.
\end{lem}

% 将U中的每个节点的feasible Label设为所有与其相连边的最大值，V中每个节点的feasible Label设为0

\begin{algorithm}[t]
	\caption{The Matching Process of KM Algorithm}
	\label{alg:match}
	\begin{algorithmic}[1]
		\Input the bipartite graph $G(U,V,E)$, the origin vertex $u$ of the augment path, the current  matching edge set $\Pi$
		\Output the updated $\Pi$
		\Function{match}{$G(U,V,E)$, $u$, $\Pi$}
		\State \textbf{for} each $v \in V$ \textbf{do} $slack(v) \gets \infty$
            %\State $slack_{v\in V} \gets\infty$
		\While{$u$ is not matched}
            \State \textbf{for} each $v \in U \cup V$ \textbf{do} $vis(v) \gets \text{false}$	
            %\State $vis_{u\in U,v\in V}\gets\text{false}$
		\If{FINDPATH($u$, $\Pi$)}
		\State break
		\EndIf
		\State $\Delta\gets\min_{v\in V \wedge vis(v)=false}\{slack(v)\}$
		\For{each $u\in U \wedge vis(u)=true$}
		\State $ex_u\gets ex_u-\Delta$
		\EndFor
		\For{each $v\in V$}
		\If{$vis(v) = true$}
		\State $ex_v \gets ex_v + \Delta$
		\Else
		\State $slack(v) \gets slack(v) - \Delta$
		\EndIf
		\EndFor
		\EndWhile
		\State \Return $\Pi$
		\EndFunction
	\end{algorithmic}
\end{algorithm}

The KM algorithm first initializes the label of each vertex, which is a preliminary estimation of the potential contribution of that vertex to the total weight of the matching. For each vertex in $U$ (resp. $V$), the label is set to the maximum weight of all edges connected to it (resp. 0). %, %and set 0 to those in $V$.
%setting $ex_{u} = \max_{v\in V, e_{uv}\in E}\{w'(e_{uv})\}$ for each vertex $u \in U$, and $ex_v = 0$ for each vertex $v$. 
Then, the algorithm traverses all vertices in $U$, attempting to match each vertex $u \in U$ by identifying an augmenting path in the equivalent subgraph originating from $u$. \ralg{match} shows the matching process. If an augmenting path is found, each vertex belonging to $U$ in this path is matched with a corresponding vertex in $V$. Specifically, for an augmenting path $\{u_1,v_1,u_2,v_2,...,u_k,v_k\}$, each vertex $u_i$ is matched with vertex $v_i$ ($i \in \{1,...,k\}$). Otherwise, adjust the labels to create more edges in the equivalence subgraph. This is done by finding the smallest label difference for edges not in the equivalent subgraph but connecting to the alternating paths we are exploring (line 8). Subtract this value from the labels of vertices in $U$ in the paths (line 10) and add it to the labels of vertices in $V$ in the paths, as well as to vertices in $V$ outside the paths but connected by an edge to a vertex in the paths (line 14). Moreover, the KM algorithm designs a relaxation function $slack$ to optimize the efficiency of calculating $\Delta$ (lines 8 and 16). By applying the $slack$ function, the time complexity of calculating $\Delta$ can be reduced from $O(n^2)$ to $O(n)$, and the time complexity of the entire KM algorithm can be reduced from $O(n^4)$ to $O(n^3)$.

\subsection{Multi-way Karmarkar-Karp Algorithm}
\begin{comment}
% 对于初始解的构造，在MPLS中采用的方法是首先令所有分区为空，然后将确定匹配边的顶点按权值从大到小排序，选择当前权重和最小的分区放入。
For the construction of the initial solution, the method used in MPLS is to first make all partitions empty, and then the vertices determining the matching edges are sorted in descending order of their weights, and the partition with the current weight and the smallest is selected to be put in.
% 这个方法虽然能在多项式时间内得到一个初始解，但解的质量不够优，会影响后续启发式算法的收敛速度以及最终的优度。
While this method yields an initial solution in polynomial time, the quality of the solution is not good enough to affect the speed of convergence of the subsequent heuristic algorithms as well as their eventual superiority.

% 如果在分区中没有节点数量的限制，这个问题和k-way number partition问题是一样的，该问题的构造算法Karmarkar-Karp不是逐个节点构建单一解决方案，而是不断解决子问题并将子问题合并知道得到最终的解，能更全面的评估分区方案并作出决策。
If there is no limit on the number of vertices in the partition, this problem is the same as the multi-way number partition problem, the construction algorithm for this problem, multi-way Karmarkar-Karp, instead of constructing a single solution node by node, keeps solving subproblems and combining them to know the final solution, which can evaluate the partitioning scheme more comprehensively and make decisions.

% 接下来先描述k-way Karmarkar-Karp算法的实现细节，然后再说明如何用该算法生成partition phase的初始解
Next, we will first describe the implementation details of the multi-way Karmarkar-Karp algorithm, and then explain how to use this algorithm to generate initial solutions for the partition phase.

\end{comment}
% \subsubsection{multi-way Karmarkar-Karp algorithm}

\begin{table*}[t]
\renewcommand{\arraystretch}{1.3}
\caption{An Example of the Multi-Way Karmarkar-Karp Algorithm.}
\label{tbl:kk_example}
\centering
\begin{tabular}{|c|cc|cccccc|}
\hline
$i$ & T1 & T2 & T3 & T4 & T5 & T6 & T7 & T8 \\
\hline
1 & (26, 0, 0) & (22, 0, 0) & (19, 0, 0) & (13, 0, 0) & (8, 0, 0) & (4, 0, 0) & (3, 0, 0) & (2, 0, 0) \\
\hline
2 & \textbf{(26, 22, 0)} & (19, 0, 0) & (13, 0, 0) & (8, 0, 0) & (4, 0, 0) & (3, 0, 0) & (2, 0, 0) & - \\
\hline
3 & (13, 0, 0) & (8, 0, 0) & \textbf{(7, 3, 0)} & (4, 0, 0) & (3, 0, 0) & (2, 0, 0) & - & - \\
\hline
4 & \textbf{(13, 8, 0)} & (7, 3, 0) & (4, 0, 0) & (3, 0, 0) & (2, 0, 0) & - & - & - \\
\hline
5 & \textbf{(6, 4, 0)} & (4, 0, 0) & (3, 0, 0) & (2, 0, 0) & - & - & - & - \\
\hline
6 & (3, 0, 0) & \textbf{(2, 0, 0)} & (2, 0, 0) & - & - & - & - & - \\
\hline
7 & \textbf{(3, 2, 0)} & (2, 0, 0) & - & - & - & - & - & - \\
\hline
8 & \textbf{(1, 0, 0)} & - & - & - & - & - & - & - \\
\hline
\end{tabular}
\end{table*}

% multi-way Karmarkar-Karp algorithm用于解决将有$n$个整数的集合$S$划分为$k$个子集，使得子集和的极差最小的问题。
% KK算法用k-tuple表示分区方案
The KK algorithm iteratively solves the Multi-way Number Partition (MNP) problem that aims to minimize the range of partitions \cite{NPAll-1}. %, \JZ{which aims to ... (like introduce MWPM in the previous section)}. 
Suppose $n$ numbers need to be divided into $k$ partitions. 
% 分区的方案用一个k-tuple list表示，每个tuple相同位置的元素的和即为最终对应partition的weight。
The partitioning scheme of the KK algorithm is represented by a list of $k$-tuples, where the sum of elements in the same position across all tuples equals the weight of the respective partition. 
The tuple list is arranged in descending order based on the maximum element in each tuple.
% The goal of MNP is to minimize the range of the partition weights.
%\JZ{Describe the goal of MNP clearly}. 
% \JZ{What is the purpose of the following sentence?} 
% The tuple list is arranged in descending order based on the maximum element in each tuple. 
% 其用k元组的列表表示状态，每个k元组的每一维表示一个子集的和，如当k=3时，3元组$(13,12,10)$表示3个子集的和分别是13,12和10。由于该问题中我们只关心子集和之间的差值，因此可以将元组中所有数值减去该元组中的最小值，即$(13-10,12-10,10-10)=(3,2,0)$。每个元组按降序排列，元组列表按每个元组中的最大整数排序，也按降序排列。
%Its state is represented by a list of k-tuples, where each element of a k-tuple corresponds to the sum of elements in a subset. 
% For example, with $k = 3$, the 3-tuple $(13, 12, 10)$ signifies that the weights of the three partitions are 13, 12, and 10, respectively. Since this problem focus on the differences between these partition sums, the tuple can be normalized by deducting the smallest value from all elements, transforming $(13, 12, 10)$ into $(3, 2, 0)$ after subtracting 10. The tuple list is arranged in descending order based on the maximum element in each tuple.

% 最初，为S中的每个整数创建一个k元组。该元组的第一个整数它的数值，其余的整数都设置为0。这些元组对应于将S的所有整数都放到第一个子集中，而其他子集中都没有整数。
Initially, a $k$-tuple is created for each given number, with the corresponding number as its first element and the remaining elements set to 0. These tuples %represent the allocation of 
indicate that all numbers are allocated into the first partition.
Subsequently, $n-1$ iterations are performed. In each round, the first two $k$-tuples in the list are popped and merged into a single tuple. The elements of these tuples are sorted in ascending and descending orders, respectively, before being added together at corresponding positions. This sorting aims to minimize the variance within the resulting $k$-tuple. For example, combining the 4-tuples $(5,3,2,0)$ and $(4,4,3,0)$ yields $(5+0, 3+3, 2+4, 0+4) = (5,6,6,4)$. Since the problem focuses on the differences between these partition sums, the tuple can be normalized by deducting the smallest value from all elements, i.e., $(5,6,6,4)$ can be normalized into $(1,2,2,0)$ and added back to the list of tuples. After completing $n-1$ rounds, only one tuple remains, representing the final normalized partitioning result.
% 具体的划分子集方案可以通过合并过程得到
A specific partitioning scheme can be derived from the merging process.

% 接下来是一个multi-way Karmarkar-Karp algorithm的完整例子。考虑初始为集合$S={26,22,19,13,8,4,3,2}$，需要将其划分为3个子集。如表格\rtbl{kk_example}所示，为multi-way Karmarkar-Karp algorithm的算法流程。
%The following presents an example of the KK algorithm. 
We present an example to illustrate the KK algorithm's iteration process in \rtbl{kk_example}, which aims to divide 
%Consider dividing 
the number set $S = \{26, 22, 19, 15, 8, 4, 3, 2\}$ into 3 partitions. %The algorithm's iteration process is illustrated in \rtbl{kk_example}. 
%The first column, labeled Iter, 
%Column $i$ indicates the current iteration round, and each row sequentially lists the tuple list in the corresponding round.
Column $i$ represents the tuple list arranged in order for the $i^{th}$ iteration round.
% 初始时$i$=1，为集合中每整数生成一个3元组并将元组按最大的数进行排序。接下来每一轮，选择$T1,T2$列对应的元组进行合并，合并后重新插入元组列表，新插入的元组在下一行中加粗表示。如当$i$=2是，将元组$(26,22,0)$和$(19,0,0)$进行合并，得到$(26,22,19)=(7,3,0)$，即为$i$=3行中$T3$列的元组。当只剩一个元组时合并结束，元组$(1,0,0)$代表划分方案$\{\{26,4,3\}, \{19,13\}, \{22,8,   2\}/}$，和分别为${33, 32, 32}$。
Initially, when $i = 1$, a 3-tuple is generated for each number. In each subsequent round, the tuples corresponding to columns $T1$ and $T2$ are popped for merging. The merged tuple is then reinserted into the list, which is highlighted in bold in the next row. For example, when $i = 2$, the tuples $(26,22,0)$ and $(19,0,0)$ are merged to form $(26,22,19) = (7,3,0)$, which is the tuple in column $T3$ for row $i = 3$. The merging process concludes when only one tuple remains. The final tuple $(1,0,0)$ represents the partition scheme $\{\{26,4,3\}, \{19,13\}, \{22,8,2\}\}$, with the sum of each partition as $\{33, 32, 32\}$ respectively.

\begin{comment}
% 伪代码可以不用
The algorithm for multi-way Karmarkar-Karp algorithm is shown in \ralg{kk}

\begin{algorithm}
	\caption{The multi-way Karmarkar-Karp algorithm}
	\label{alg:kk}	
	\begin{algorithmic}[1]
		\Input $S,k$
		\Output $\{P_1, P_2,...,P_k\}$
		\Function{Karmarkar-Karp}{$S,k$}
        \State $TupleList \gets$ Initial tuples of $S$
        \State $MergeInfo \gets \emptyset$  
        \While{$TupleList$'s size $>$ 2}
            \State $T1, T2 \gets$ pop the first two tuples from the $TupleList$
            \State $TNew \gets$ MergeTuple($T1$, $T2$, $MergeInfo$)
            \State Insert $TNew$ into $TupleList$
        \EndWhile
        \State $\{P_1, P_2,...,P_k\} \gets$ Obtain partition result based on $MergeInfo$
		\State \Return $\{P_1, P_2,...,P_k\}$
		\EndFunction
	\end{algorithmic}
\end{algorithm} 
\end{comment}

\begin{comment}
%考虑节点数量限制
\subsubsection{Consider the limitation on the number of vertices}

% 通过元组的合并过程得到分区方案后，若某个分区的顶点数超过数量限制则将其随机放置在一个仍有容量的分区。后续的用启发式搜索阶段会对解进行进一步优化。
After obtaining the partitioning scheme through the tuple merging process, if the number of vertices in a partition exceeds the quantity limit, they are randomly placed in another partition that still has capacity. Subsequent heuristic search stages will further optimize the solution.
\end{comment}

%\section{The proposed FMP\_HEE Algorithm}
\section{The Proposed Algorithm}
\label{subsec:alg}
In this section, we provide a detailed overview of the proposed Fast Iterative Match-Partition Hybrid Genetic Algorithm (FIMP-HGA). Initially, we outline the general scheme of FIMP-HGA, followed by an in-depth exploration of its search components.

\subsection{General Scheme}

% The FMP$_\text{HEE}$ algorithm 的主要创新点在匹配阶段寻找最优解的精确算法以及匹配阶段使用的hybrid evolutionary algorithm with elite strategy。同时每轮迭代后对图$G'$进行适当的修改得到新的匹配方案，提供改进解的可能。在算法实现中采用增量更新的方式得到当前图$G'$的匹配。每当有边加入或删除都通过KM-M算法更新图$G'$的匹配结果。

As outlined in \Rsec{intro}, the FIMP-HGA consists of three stages: % the match stage, partition stage, and graph modification stage. 
match, partition, and graph modification.
The primary innovations of FIMP-HGA include the incremental update approach for finding the optimal matching solution during the match stage and the hybrid genetic algorithm with an elite strategy during the partition stage. 
In the final stage of each iteration, the graph is strategically modified to facilitate potentially better matching schemes. 
% In its implementation, the algorithm adopts an incremental update approach to determine the current matching of the graph $G'$. Whenever an edge is added or removed, the KM-M algorithm is employed to update the matching result of the graph $G'$.

% 首先对变量进行初始化（第2到4行），包括原图G的匹配方案，边的禁忌列表等。接下来通过迭代更新最优解（第5到29行）。当若干轮没有更新最优解或多次执行recover edge操作后，说明当前最优解难以改进，结束算法（第5行）。每轮迭代首先根据匹配方案获取partition阶段每个顶点的权值（第6行），然后通过HEE算法求得划分问题的解（第7行）。根据划分方案获得本轮迭代需要禁忌的边（第8行），并初始化解禁列表（第9行）。若当前为第一轮迭代或当前解优于最优解则更新（第10到12行）。否则若当前解离最优解差距较大，说明当前禁忌的边影响了解的质量，将部分边加入解禁列表（都15至18行）。若禁忌列表中有边达到禁忌轮次，同样放入解禁列表中（第20至21行）。最后在图$G'$中恢复解禁列表中的边以及删除本轮需要禁忌的边，每次操作均通过KM-M算法更新匹配方案（第22至28行）。迭代结束后返回得到的最优解（第30行）。

% 在算法的对应位置加入小节信息

The FIMP-HGA is depicted in \ralg{FIMP-HGA}.
The algorithm first performs some initialization, %is performed (lines 2 to 4), 
including the matching scheme of the original graph $G$, the tabu list $L_{t}$ for edges, etc. The solution space is then explored through iterations (lines 5 to 28),  terminating when the number of iterations without updating the best-found solution $t$ or the number of times executing edge recovery $t_{ER}$ reaches the set values $T$ and $T_{ER}$. %, the iteration terminates (line 5). 
%Each iteration begins by obtaining the weight 
In each iteration, the algorithm first obtains the weight of each vertex based on the matching scheme $\Pi$ (line 6), then solves the partition stage using the HGA (\rsec{HGA}) (line 7). Based on the obtained solution, a tabu edge $e_{t}$ is selected (line 8), and the best-found solution is attempted to be updated (lines 10 to 13).
%select the tabu edge $e_{t}$ (line 8) and attempt to update the optimal solution (line 10 to 13). 
If the current solution is significantly worse than the best-found solution, it indicates that the current tabu edges degrade solution quality, so certain tabu edges are selected for removal and added to the release list $L_{r}$ based on the edge recovery strategy (\rsec{ER}) (lines 15 to 18). The edges that reach the tabu tenure are also added to the release list (line 20). Finally, the edges in the release list are reinserted (lines 22 to 25), and the tabu edge is removed from graph $G'$ (line 26). %, with each operation updating the matching scheme using the KM-M algorithm (\rsec{KM-M}) (lines 24, 27). 
The algorithm determines the new matching scheme in the next iteration by calling the proposed fast KM-M algorithm (\rsec{KM-M}) (lines 24 and 27).
After the iterations, the best-found solution is returned (line 29).

\begin{algorithm}[t]
    \caption{The FIMP-HGA Algorithm}
    \label{alg:FIMP-HGA}	
    \begin{algorithmic}[1]
        \Input the PMMWM instance input $G(U,V,E)$, $k$, $\bar{u}$, the maximum iterations $T$ without updates, the maximum number of edge recovery operators $T_{ER}$, the maximum iterations $T_{HGA}$ of HGA without updates
        \Output the best solution $\Pi, \mathcal{P}$ found
        \Function{FIMP-HGA}{$G$, $k$, $\bar{u}$, $T$, $T_{ER}$, $T_{HGA}$}
        \State $G' \gets G, \Pi \gets$ KM-M($G', \emptyset, \emptyset$)
        \State $L_{t} \gets \{\}, t \gets 0, t_{ER} \gets 0$
        \State $\Pi_{ans} \gets \emptyset, \mathcal{P}_{ans} \gets \{\}$
        \While{$t \leq T$ \AND $t_{ER} \leq T_{ER}$}
            % \State $\Pi \gets$ match result from KM-M
            \State $W \gets$ obtain vertex weights based on $\Pi$
            \State $\mathcal{P} \gets$ HGA($W, k, \bar{u}, T_{HGA}$)
            \State $e_{t} \gets $ select the tabu edge base on $\Pi,\mathcal{P}$
            \State $L_{r} \gets \emptyset$
            \If{$\Pi_{ans} = \emptyset$ \OR $f(\Pi, \mathcal{P})<f(\Pi_{ans}, \mathcal{P}_{ans})$}
                \State $\Pi_{ans} \gets \Pi, \mathcal{P}_{ans} \gets \mathcal{P}$
                \State $t \gets 0$
            \Else
                \State $t \gets t + 1$
                \If {$f(\Pi, \mathcal{P}) \gg f(\Pi_{ans}, \mathcal{P}_{ans})$}
                    \State $L_{r} \gets L_{r} \cup$ ER($L_{t}$)
                    \State $t_{ER} \gets t_{ER} + 1$
                \EndIf
            \EndIf
            % \State $e \gets$ pop edge from $tabuList$
            \State $L_{r} \gets L_{r} \cup$ \{edges that reach tabu iteration\}
            \State $L_{t} \gets L_{t} \backslash L_{r}$ 
            \For{each $e_{uv} \in L_{r}$}
                \State add edge $e_{uv}$ to $G'$
                \State $\Pi \gets$ KM-M($G', \Pi, e_{uv}$)
            \EndFor
            \State remove $e_{t}$ from $G'$, $L_{t} \gets L_{t} \cup e_{t}$
            \State $\Pi \gets$ KM-M($G', \Pi, e_{t}$)
        \EndWhile
        \State \Return $\Pi_{ans}, \mathcal{P}_{ans}$
        \EndFunction
    \end{algorithmic}
\end{algorithm}

\subsection{Fast Matching Algorithm KM-M}
\label{subsec:KM-M}
The proposed KM-M algorithm %used to solve the MSWM problem 
is used to tackle the match stage efficiently by retaining the previous matching information %from each iteration 
to avoid redundant computation. Whenever an edge $e_{uv}$ is added or removed in graph $G'$, FIMP-HGA invokes KM-M to update the optimal matching results. Compared to the KM algorithm, the KM-M algorithm reduces the complexity of the match stage during iterations from $O(n^3)$ to $O(n^2)$.
In the following, we first introduce details of the KM-M algorithm and subsequently prove its correctness.

% 用不用加algorithm？还是直接用KM-M

\subsubsection{Details of the KM-M Algorithm}

% KM-M算法通过将每轮迭代匹配阶段的信息保留，避免冗余的计算。每当图$G'$中的新增或删除边时
% The KM-M retains the matching information from each iteration to avoid redundant computations. As a result, the proposed KM-M can reduce the time complexity of the first stage in each iteration except the first iteration of KM from $O(n^3)$ to $O(n^2)$. Whenever an edge $e_{uv}$ is added or removed in the graph $G'$, FMP\_HEE will send $G'$, $\Pi$ and $e_{uv}$ to the KM-M algorithm. 
% The KM-M algorithm only needs to remove the matching of $u$ in $\Pi$, and rematch $u$ in $G'$, which is significantly more efficient than performing KM on the entire $G'$. 
% 如果传入匹配$\Pi$方案为空，则当前为第一轮迭代。与\ralg{km}一样，初始化后对进行完整的匹配流程（第2到7行）。否则，对当前的匹配结果进行增量修改来快速得到匹配方案。

\begin{algorithm}[t]
	\caption{The KM-M Algorithm}
	\label{alg:km-m}	
	\begin{algorithmic}[1]
		\Input the bipartite graph $G(U,V,E)$, the matching scheme from the previous round $\Pi$, the modified edge $e'_{uv}$
		\Output the updated $\Pi$
		\Function{KM-M}{$G(U,V,E)$, $\Pi$, $e'_{uv}$}
		\If{$\Pi = \emptyset$}
            \State \textbf{for} each $v \in V$ \textbf{do} $ex_{v} \gets 0$
            \For{each $u\in U$}
            \State $ex_{u} \gets \max_{v\in V \wedge e_{uv}\in E}\{w'(e_{uv})\}$
            \EndFor
            %\State $ex_{v \in V} \gets 0,ex_{u\in U} \gets \max_{v\in V, e_{uv}\in E}\{w'(e_{uv})\}$
            \State $\Pi \gets \emptyset$
    		\For{$u \in U$}
    		\State $\Pi \gets$ MATCH($G$, $u$, $\Pi$)
    		\EndFor
		\Else
		  % \State $u, v \gets$ the two vertices connected by edge $e$
            % \State $w \gets$ the weight of edge $e$
            % \State $e_{uv} \gets$ the edge in $\Pi$ and is connected with $u$ 
            \State $ex_{u} \gets \max\{ex_{u}, w'(e'_{uv}) - ex_{v}\}$
            % \State $e_{u} \gets$ the matching edge of $u$ in $\Pi$
    	  \State $\Pi \gets$ MATCH($G$, $u$, $\Pi \backslash$\{the matching edge of $u$\})
		\EndIf
		\State \Return $\Pi$
		\EndFunction
	\end{algorithmic}
\end{algorithm}

The KM-M algorithm is shown in \ralg{km-m}. 
If the input matching scheme $\Pi$ is empty (the first iteration, line 2 in \ralg{FIMP-HGA}), the complete KM algorithm is conducted (lines 3 to 10). Otherwise, the current matching result $\Pi$ is incrementally modified to quickly obtain a new matching scheme (lines 12 to 13). If the weight $w'(e'_{uv})$ is greater than $ex_u + ex_v$, then set $ex_u$ to $w'(e'_{uv}) - ex_v$ (line 12). Consequently, remove $u$'s matching edge $e_u$ from $\Pi$, and re-match $u$ by the MATCH process (\ralg{match}) used in the KM algorithm (line 13). Edge removal is achieved by setting its weight to a sufficiently small value, ensuring the bipartite graph structure remains unchanged.

\rfig{km-m} is an example of the KM-M algorithm. \rfig{km-m1} illustrates a maximum weight perfect matching $\Pi$ in graph $G$ calculated by the KM algorithm. The red and blue edges are in the equivalence subgraph (satisfy $ex_u+ex_v=c_{uv}$), and the red edges constitute the $\Pi$. Suppose we remove edge $e_{u_1 v_1}$ by modifying its weight from $-2$ to $-500$ and yield a new graph $G'$. The KM-M algorithm performs as follows. 
Firstly, as \rfig{km-m2} shows, it removes the matching of $u_1$ from $\Pi$ and the modified edge $e_{u_1 v_1}$ from the equivalence subgraph. Since the modified edge weight is smaller than the original edge weight, there is no need to change $ex_u$.
Then it uses the matching algorithm (\ralg{match}) to find a matching edge for $u_1$. As \rfig{km-m3} shows, the algorithm adjusts the feasible label to add edge $e_{u_1 v_2}, e_{u_3 v_1}$ into the equivalence subgraph. 
Finally, as shown in \rfig{km-m4}, it finds an augmenting path $\{u_1,v_2,u_3,v_1\}$ (the dashed edges) and replaces the edge $e_{u_3 v_2}$ with $e_{u_1 v_2},e_{u_3 v_1}$ in the matching $\Pi$. The new matching is as good as the output of the KM algorithm on the entire graph $G'$. Obviously, the KM-M algorithm reduces a lot of redundant operations for calculating the matching of other vertices.

\begin{figure}[t]
    \centering
    \begin{subfigure}[b]{0.24\textwidth}
        \centering
		\setlength{\abovecaptionskip}{0.cm}
		\includegraphics[width=\linewidth]{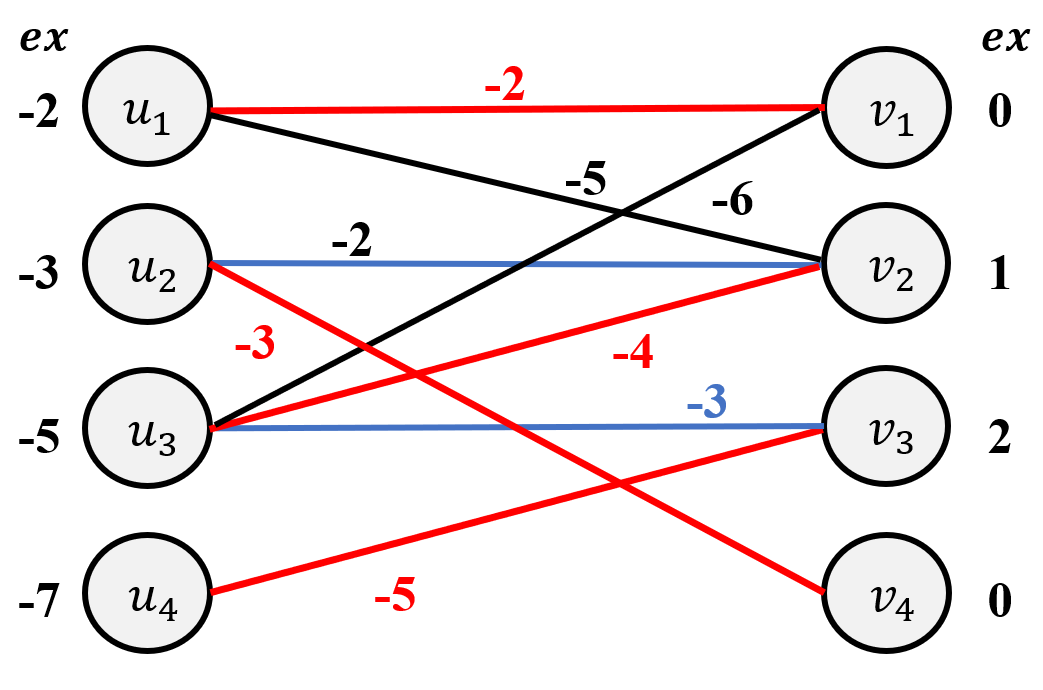}
		\caption{A matching $\Pi$ on $G$}
		\label{fig:km-m1}
    \end{subfigure}
    \hfill
   \begin{subfigure}[b]{0.24\textwidth}
        \centering
		\setlength{\abovecaptionskip}{0.cm}
		\includegraphics[width=\linewidth]{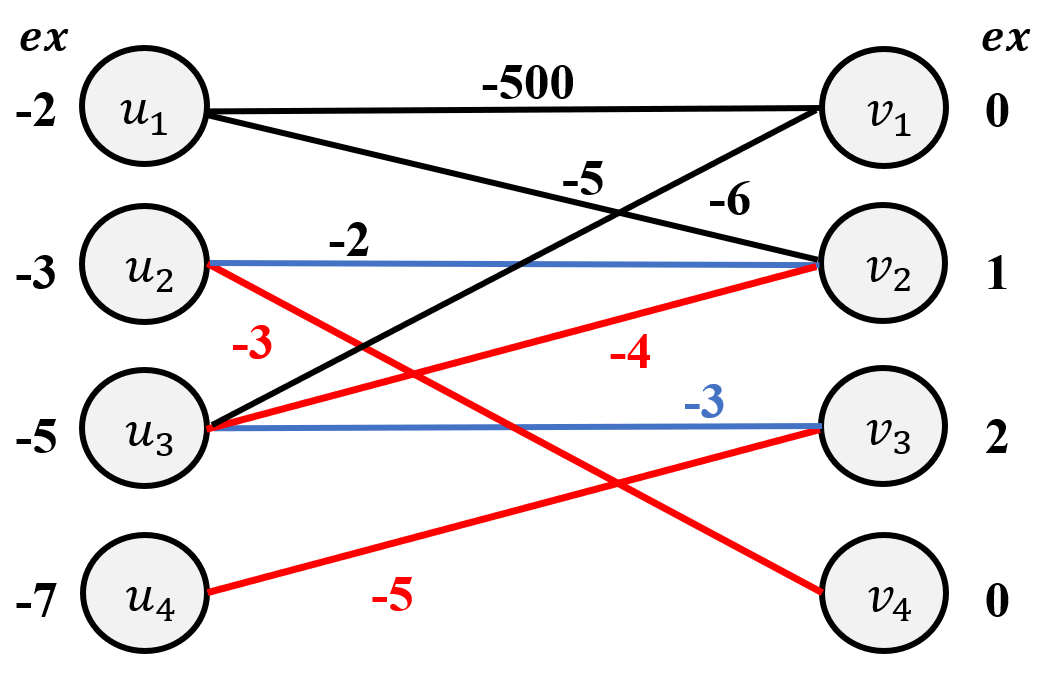}
		\caption{Modify the weight of an edge}
		\label{fig:km-m2}
    \end{subfigure}
    \begin{subfigure}[b]{0.24\textwidth}
        \centering
		\setlength{\abovecaptionskip}{0.cm}
		\includegraphics[width=\linewidth]{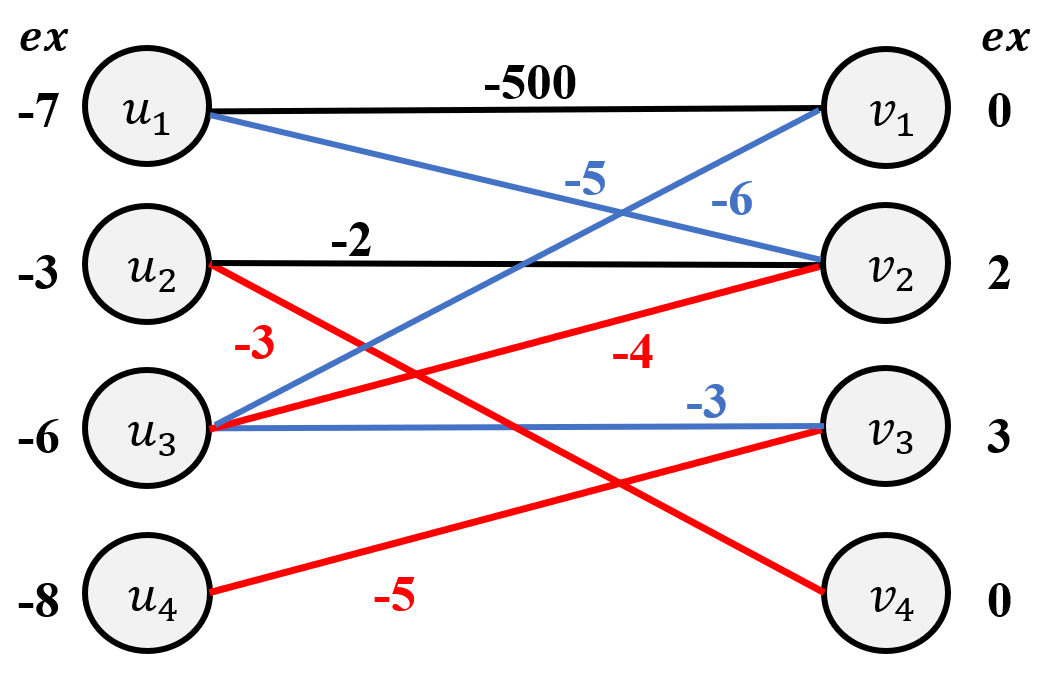}
		\caption{Adjust the feasible label}
		\label{fig:km-m3}
    \end{subfigure}
    \hfill
   \begin{subfigure}[b]{0.24\textwidth}
        \centering
		\setlength{\abovecaptionskip}{0.cm}
		\includegraphics[width=\linewidth]{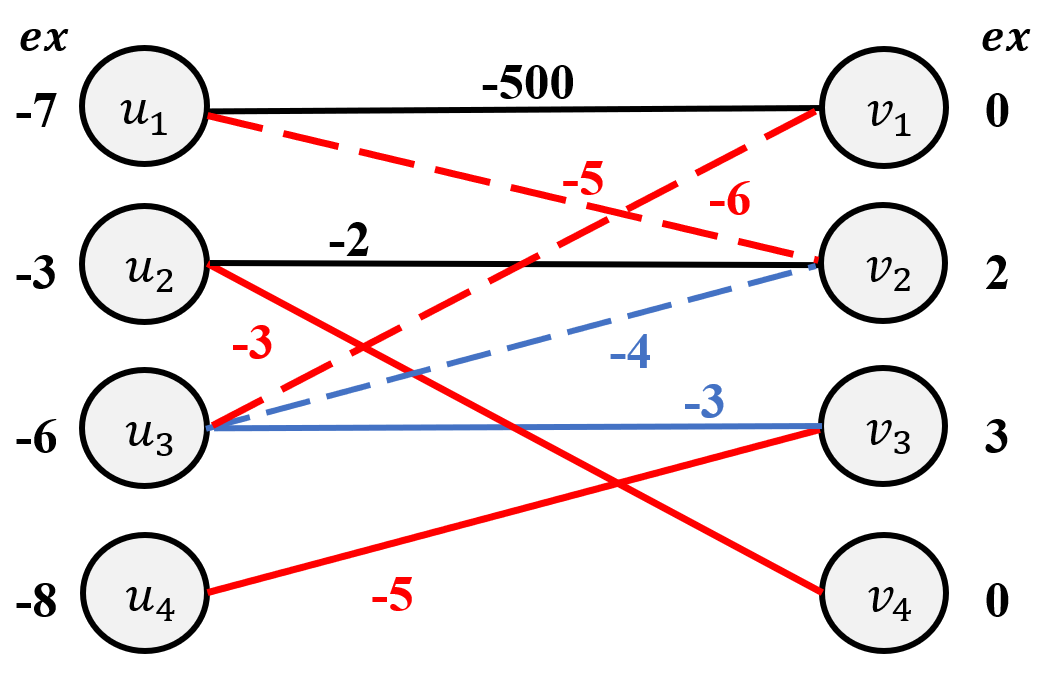}
		\caption{Find an augmenting path}
		\label{fig:km-m4}
    \end{subfigure}
\caption{An illustrate of the KM-M algorithm.}
\label{fig:km-m}   
\end{figure}

%Since the improved algorithm is to reduce the matching overhead of the KM algorithm	, the improved algorithm is referred to as MP$_{\text{KM-M}}$ in this paper.

\subsubsection{The Correctness of the KM-M Algorithm}

We validate the correctness of the KM-M algorithm by the primal-dual method~\cite{DBLP:books/daglib/0069809}. 
The dual of the linear program (LP) of the MWPM problem is to minimize the sum of the feasible label. 
% According to previous studies \cite{LP1,LP2}, if the primal and dual solutions adhere to the complementary slackness condition, they are both optimal solutions for their respective LPs. 
Building on insights from previous research \cite{LP1,LP2}, it is established that when both primal and dual solutions comply with the complementary slackness condition, they are deemed optimal for their respective LP formulations.
The complementary slackness condition for this problem stipulates that each edge $e_{uv} \in \Pi$ satisfies $ex_u + ex_v = w'(e_{uv})$, which is equivalent to the existence of a perfect matching in the equivalent subgraph. To find this solution, the KM algorithm starts with a feasible dual solution and decreases the sum of the feasible label to add more edges into the equivalent subgraph until a perfect matching exists.
% i.e., $\Pi$ is a perfect matching in the equivalence subgraph of the dual solution.

%We maintain a feasible dual solution. Feasibility means we ensure that $ex_u+ex_v \ge c_{uv}$ for all $e_{uv} \in E$. We try to find a primal solution that satisfies complementary slackness with respect to the current dual solution. We either find a perfect matching, or we get a vertex cover of size less than $n$. If we cannot find such solution, we find a direction of dual increase. The vertex cover in the bipartite matching instance corresponds $U^{\prime}, V^{\prime}$, a subset of the $U,V$, such that $|U^{\prime}| + |V^{\prime}| < n$ and if $ex_u+ex_v = c_{uv}$ then $u \in U^{\prime}$ or $v \in V^{\prime}$. We let $\Delta = \min_{u \in U, v \in V, e_{uv} \in E}\{ex_u+ex_v-c_{uv}\}$, and we update $u_i = u_i - \Delta$ for all $u \notin U^{\prime}$ and $v_i = v_i + \Delta$ for all $v \in V^{\prime}$. It can be proved new dual is feasible and the sum of feasible label strictly increase.

% Now let's get back to the correctness of KM-M. We first present the theorem, and then the proof.
% Now let's get back to prove the correctness of KM-M. 
Validating the correctness of the proposed KM-M algorithm is equivalent to proving the following theorem.

% \begin{comment}

\begin{comment}
\begin{thm}
\label{thm:km-m}
Given a bipartite graph $G = (U,V,E)$ and a maximum weight perfect matching $\Pi$ calculated by the KM algorithm. When we modify the weight of exactly one edge $e_{uv} \in E$, we can set $ex_u$ to $max\{ex_u,w'(e_{uv}) - ex_v\}$ 
% if the new weight $w'(e_{uv})$ is greater than $ex_u + ex_v$, 
and remove $e_{uv}$ from $\Pi$. In this way, %continuing the KM algorithm 
re-match $u$ with the matching method in the KM algorithm will not affect its optimality.
\end{thm}
\end{comment}

\begin{thm}
\label{thm:km-m}
Given a bipartite graph $G = (U,V,E)$ and a maximum weight perfect matching $\Pi$ of $G$ calculated by the KM algorithm. Suppose $G'$ is the graph obtained by modifying the weight of exactly one edge $e'_{uv}$ in $G$, and $w'(e'_{uv})$ is updated to the new weight of $e'_{uv}$ in $G'$. Let $e_u$ denote the matching edge of $u$. Setting $ex_u = \max\{ex_u,w'(e'_{uv}) - ex_v\}$ and performing the KM algorithm on $G'$ continuously from $\Pi \backslash$\{$e_u$\} lead to an optimal matching of $G'$.
\end{thm}

% \end{comment}

%如果新的权重$w(e_{uv}) > ex_u + ex_v$，则令$ex_u=w(e_{uv})-ex_v$

\begin{proof}

This is equivalent to proving the feasibility of the dual solution after modification, that is, $ex_{u}+ex_{v} \ge w'(e'_{uv})$.
% 如果修改的边权值变小了，由于当前$ex_{u}+ex_{v}= w'(e_{uv})$，因此满足feasibility。若修改的边权增大，KM-M会set $ex_u$ to $w(e_{uv}) - ex_v$，同样满足因此满足feasibility。
% 当$e_{uv}$权值减小时，由于$ex_{u}$和$ex_v$不变，仍然满足feasibility.
When the weight of $e'_{uv}$ decreases, since $ex_{u}$ and $ex_v$ remain unchanged, feasibility is still satisfied. Otherwise, if the weight increases, $ex_u$ will be adjusted to $w'(e'_{uv}) - ex_v$, thereby maintaining feasibility.
%Because we only modify the weight of edge $e_{uv} \in \Pi$ , and we known $ex_{u}+ex_{v}= -c_{uv}$, thus $ex_{u}+ex_{v} > -c^{\prime}_{u,v} = -10^2c_{max}$ holds. 
So, the dual solution is still feasible for the current matching, indicating that the previous information can be kept. 

Another approach is to discuss different scenarios based on whether edge $e'_{uv}$ is in the equivalence subgraph. If $w'(e'_{uv})$ is greater than $ex_u + ex_v$, it indicates that the edge's weight is underestimated. So, add $e_{uv}$ to the equivalence subgraph by adjusting $ex_u$ to $w'(e'_{uv}) - ex_v$, and then re-match vertex $u$. If $w'(e'_{uv})$ decreases from being equal to $ex_u + ex_v$, then this edge needs to be removed from the equivalence subgraph, and if it is a matched edge, vertex $u$ must be re-matched. In other cases, where $e_{uv}$ is not in the equivalence subgraph before and after the modification, there is no impact on the matching process.

Through these two methods, we can conclude that \rthm{km-m} holds, and therefore the KM-M algorithm is correct.
\end{proof}

\subsection{Hybrid Genetic Algorithm with Elite Strategy}
\label{subsec:HGA}
% Hybrid evolutionary算法参考种群的思想，同时维护问题的多个解，并通过杂交算子在原有解的基础上产生新的解，从而保证算法的疏散性。
The population-based evolutionary algorithm maintains multiple individuals (i.e., solutions) simultaneously. The genetic algorithm is a representative evolutionary algorithm, which 
% is one such method, it 
generates new solutions by combining diverse individuals with the crossover operator. %, and optimizes them through local search. 
The local search algorithm finds better solutions by exploring the neighborhood of the current solution. 
Our proposed Hybrid Genetic Algorithm (HGA) combines the advantages of the genetic algorithm and local search method, exhibiting excellent search ability for the partition stage.
% solution space by searching the  
% This approach ensures a suitable balance of exploration and exploitation. 
Furthermore, an elite strategy is used in the genetic algorithm to improve the quality of the explored solution space by consistently selecting the elite individual as one of the parents in the crossover operator, thereby enhancing the algorithm's efficiency.
% 另外，genetic algorithm中的精英策略通过固定精英个体为crossover opeator其中一个parent，来提升探索的解空间的价值，从而提高算法运行效率

%The hybrid evolution algorithm draws inspiration from the concept of a population, maintaining multiple solutions to a problem simultaneously. It generates new solutions based on existing ones through the crossover operator, and optimizes them through local search. This approach ensures the diversity of the algorithm, helping to escape from local optima.
% and potentially leading to more globally optimal solutions.
% thereby ensuring the diversity of the algorithm.

% 对于上一节的Improve local search算法每次都在当前解的基础上进行修改来提升解的质量，容易陷入局部最优解。因此，我们使用Hybrid evolutionary的方式来增加算法的疏散性
% The Improved Local Search algorithm, as discussed in the previous section, modifies the current solution to enhance its quality, which can easily lead to local optima. Therefore, We employ a hybrid evolutionary algorithm with elite strategy(HEE) to increase the diversity of the search while simultaneously ensuring operational efficiency.

\begin{algorithm}[t]
	\caption{Hybrid Genetic Algorithm with Elite Strategy}
	\label{alg:HGA}	
	\begin{algorithmic}[1]
		\Input the vertex weight $W$, the PMMWM instance input $k$, $\bar{u}$, the maximum iterations $T_{HGA}$ without updates
		\Output the best partition scheme $\mathcal{P}$ found
		\Function{HGA}{$W$, $k$, $\bar{u}$, $T_{HGA}$}
        \State Initialize population $Pops$ by vertex weight $W$
        % \State Optimize the solutions in $pops$ using ImproveLocalSearch
        %\For{$sol \in pops$}
        %    \State $pops \gets pops \backslash sol$
        %    \State $pops \gets pop \cup $ LocalSearch($sol$)
        %\EndFor
        \State $t \gets 0$  
        \While{$|Pops| > 1$ \AND $t \leq T_{HGA}$}
            \State $\mathcal{P}_{elite} \gets $ best individual in $Pops$
            \State $\mathcal{P}_{mate} \gets $ another random individual in $Pops$
            \State $\mathcal{P}_{child} \gets$ Crossover($\mathcal{P}_{elite}$, $\mathcal{P}_{mate}$)
            \State $\mathcal{P}_{child} \gets$ MultilevelLocalSearch($\mathcal{P}_{child}$)
            \If {$f(\mathcal{P}_{child}) < f(\mathcal{P}_{elite})$}
                \State $t \gets 0$
                %\State $pops \gets pops \backslash sol_{e}, pops \gets pops \cup sol_{c}$ 
            \Else
                %\If {$f(sol_{c}) = f(sol_{e})$ or $f(sol_{c})=f(sol_{r})$}
                %    \State $pops \gets pops \backslash sol_{r}$
                %\Else
                %    \State $pops \gets pops \backslash sol_{r}, pops \gets pops \cup sol_{c}$ 
                %\EndIf
                \State $t \gets t + 1$ 
            \EndIf
            \State $Pops \gets $ UpdatePopulation($Pops$, $\mathcal{P}_{mate}$, $\mathcal{P}_{child}$)
        \EndWhile
		\State \Return best individual in $Pops$
		\EndFunction
	\end{algorithmic}                                      
\end{algorithm} 

% We propose a genetic algorithm with an elite strategy to solve the partition stage of PMMWM.
The overall procedure of HGA is shown in \ralg{HGA}. The $f(\mathcal{P})$ is defined to represent the objective value of the partition scheme $\mathcal{P}$.
% 算法首先初始化种群$pops$（第2行），并通过improved local searc优化初始种群中的解（第3行到第5行）。然后通过种群的演化搜索更多的解（第6行到第23行）。若一定轮次没有改进最优解或者种群中只剩一个解，则结束搜索。具体的，每次迭代分别选择种群中的精英解$sol_{elite}$和随机解$sol_{random}$（第6行到第23行为）。将这两个解通过交叉算子生成新的解$sol_{new}$（第10行），并使用improved local search进行优化（第11行）。若$sol_{new}$优于$sol_{elite}$，则在$pops$进行替换（第29到第31）。若其优度与$sol_{random}$或$sol_{elite}$相同，则认为解$sol_{random}$难以为精英解带来多样性，将其从$pops$中移除（第33到34行）。否则，$sol_{new}$仍有继续保留在pops$中增加多样性的价值，将其替换$pops$中的$sol_{random}$（第36到37行）。
HGA starts the search by initializing the population $Pops$ (\rsec{pop}) (line 2). Next, it selects individuals $\mathcal{P}_{elite}$ and $\mathcal{P}_{mate}$ for evolution based on the elite strategy (\rsec{elite}) (lines 5 to 6) and employs the proposed Greedy Partition Crossover  (GPX) operator (\rsec{gpx}) to generate new offspring $\mathcal{P}_{child}$ (line 7). 
% Subsequently, the solution quality of $\mathcal{P}_{child}$ is improved by the proposed Multilevel Local Search (MLS) algorithm.  Finally, the population is updated with $\mathcal{P}{child}$.
After optimizing $\mathcal{P}_{child}$ with the proposed Multilevel Local Search (MLS) algorithm (\rsec{mls}) (line 8), the population is updated (line 14). 
The algorithm terminates and returns the best-found solution when there is only one individual left in the population or if no better solution is updated within a certain number of iterations $T_{HGA}$.

% 通过proposed Multilevel Local Search (MLS) algorithm提高$\mathcal{P}_{child}$的solution quality。定义函数$f(\mathcal{P}_{child})$表示partition方案\mathcal{P}_{child}的目标函数值。最后通过$\mathcal{P}_{child}$更新种群。

%Following is a detailed description of the key 
Detailed descriptions of the key components of HGA are as follows.

\subsubsection{Population Initialization}
\label{subsec:pop}

% 种群的初始解需要保证多样性，因此采用多种方式生成初始解：（1）multi-way Karmarkar-Karp算法（2）采用贪心算法，将顶点权重按从大到小依次放入分区。每次选择当前权重最小的分区放置（3）在贪心算法的基础上加入随机性，当顶点选择放入的partition时，有一定概率不选择当前权重最小的分区，而是随机选择一个容量为到达上限的分区。这个概率设为50%。

HGA generates the initial population with the following methods to ensure both diversity and quality.

\begin{itemize}
\item{\textbf{Multi-way Karmarkar-Karp algorithm (KK):} obtaining the partitioning scheme through the tuple merging process. If the number of vertices in any partition exceeds its capacity limit, these excess vertices are then randomly allocated to other partitions that have remaining capacity.}
\item{\textbf{Greedy algorithm (GD):} vertices are sorted in descending order of their weights, and each vertex is sequentially placed into the partition with the smallest weight. When selecting, only partitions that have not reached the capacity limit are considered, and this principle applies to the random greedy algorithm as well.}
\item{\textbf{Random Greedy algorithm (RGD):} this approach adds randomness to the greedy algorithm. When assigning a vertex, rather than always opting for the partition with the smallest weight, there is a 50\% chance of randomly selecting a partition.}
\end{itemize}

% (1) multi-way Karmarkar-Karp algorithm (2) greedy algorithm: vertex weights are sorted in descending order, and each vertex is sequentially placed into a partition. The vertex is always placed in the partition with the current smallest weight sum to maintain balance. (3) random algorithm: this approach adds an element of randomness to the greedy algorithm. When assigning a vertex, rather than always opting for the partition with the smallest weight sum, there's a 50\% chance of randomly selecting a partition that still has available capacity.

% 通过方法一和方法二各生成一个解，其余解通过方法三生成。在保证种群优度的同时，通过方法三引入解的不同特征，满足多样性的需求。
The KK and GD methods each generate a solution to uphold the population's quality. The remaining solutions are produced using the RGD method, introducing varied characteristics to ensure population diversity. Before integration into the population, each generated solution is optimized with MLS. Typically, the solution generated by the KK method assumes the role of the elite individual.

% 经过测试，在算法中，令种群的大小为5。能够同时保证效率和优度。
% After testing, it has been determined that setting the population size to 5 in the algorithm effectively balances efficiency and solution quality.

\subsubsection{Greedy Partition Crossover Operator}
\label{subsec:gpx}

% 交叉算子用于模拟生物遗传中的交叉过程，根据两个父代解生成新的个体。结合两个父代解中的特征，创造新的特征组合并引入一定的随机性，保证算法的多样性。
The crossover operator is used to simulate the evolutionary process in biological genetics, passing problem-specific genetic knowledge from parents to offspring solutions. By combining features from both parents, the crossover creates new feature combinations and introduces a degree of randomness. This ensures diversity in the population, mirroring the genetic diversity in natural selection and evolution.

% 我们采用了解决图着色问题中用到的Greedy Partition Crossover(GPX)方法，在两个父代中交替的选择当前含有节点数量最多的Partition，若有多个则随机选择一个，加入到子代的解中。已经加入自带的顶点后续不再考虑，可以视为在父代中删除。直到进行$k$轮后，可能仍有一些节点未加入子代的解中，则随机的加入未到达容量限制的partition。
We propose the GPX operator for the partition stage. A similar crossover operator is also employed in the graph coloring problem \cite{GCP}. In the GPX operator, we alternately choose a partition containing the highest number of vertices from each parent solution and incorporate this partition into the child solution. 
%If there are multiple largest partitions, choose one randomly. 
    Vertices that have been added are not considered in subsequent steps, thereby treating them as removed from the parents. After the $k$ rounds of the process, some vertices may still not be included in the child solution. These remaining vertices are then randomly added to any partition that has not yet reached its capacity limit. 

% 如图\rfig{CO}是一个GPX crossover的示例，先不考虑顶点的权重。对于该分区问题，顶点集合为$\{A,B,C,D,E,F,G,H,I,J\}$，分区数量为3。每一行代表一个解的分区方案，不同颜色的块代表不同分区。GPX分为四步进行，前三步中黑色框框起来的表示该轮加入子代解的分区。第四轮将剩余的顶点$D$随机放入子代解的一个分区中。
\rfig{CO} illustrates an example of our crossover operator. %, initially not considering the weights of the vertices. For this partition problem
The vertex set is $\{A, B, C, D, E, F, G, H, I, J\}$, and the partition number is 3. Each row represents a partitioning scheme of a solution, with different colored blocks indicating different partitions. The crossover process is divided into four steps. In the first three steps, the partitions added to the child's solution are indicated by black boxes. In the fourth step, the remaining vertex $\{D\}$ is randomly placed into one of the partitions in the child solution.

\begin{figure}[t]
\centering
\includegraphics[width=\linewidth]{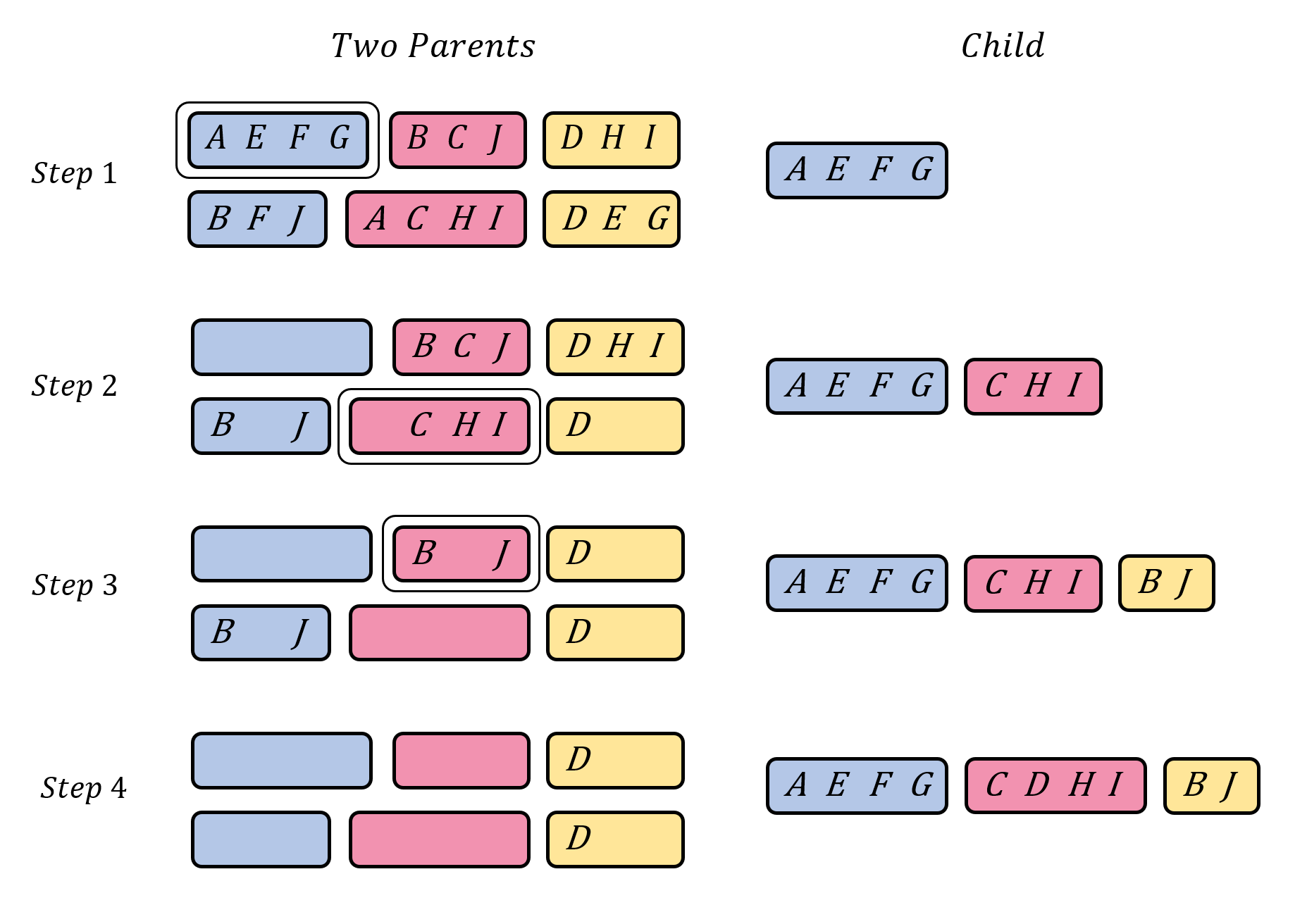}
\caption{An example of the GPX operator for a vertex set of 10 elements (A - J) and three partitions.}
\label{fig:CO}
\end{figure}

\subsubsection{Multilevel Local Search}
\label{subsec:mls}

% 局部搜索的目标是通过改变某些顶点的partition继续优化目标。令partition的权重为其中包含的顶点的权重和, partition阶段的目标是最小化partition权重的最大值，因此局部搜索主要通过将当前权重最大的partition中的顶点与其他partition进行交换来优化目标。
% 因此领域动作专注于当前权重最大的分区，尝试通过交换操作来optimizes the objective function while striving to ensure a balanced distribution of weights across the partitions.
% The goal of the local search is to further optimize the solution by altering the partition of certain vertices. 

HGA uses the MLS algorithm to intensively explore neighboring solutions by adjusting the partitions of certain vertices. 
% The objective is to minimize the maximum weight of partitions. 
Neighborhood actions in the MLS algorithm concentrate on the current partition with the largest weight, attempting to optimize the %objective function
solution quality through swap operations while aiming to maintain a balanced weight distribution across the partitions.
%Therefore, the swap actions we define are specifically targeted at vertices in the partition with the largest weight. This approach directly optimizes the objective function while striving to ensure a balanced distribution of weights across the partitions.

%Consequently, the local search primarily focuses on optimizing this objective by exchanging vertices from the partition with the highest weight with those from other partitions. This approach aims to balance the partition weights more effectively.

% 我们提出的Improved local search是一个二级的局部搜索。将局部搜索策略分为两种，首先一级局部搜索尝试将权值最大的partition和最小的partition中的顶点进行交换，使得权值和尽量平均。若该策略无法继续对目标进行优化，则使用二级局部搜索尝试将权值最大的partition和其余所有partition都进行交换。
% 这两级局部搜索尝试的交换操作也有所区别，分别注重搜索的深度和广度，后续会分别详细描述。
%We propose a three-level local search algorithm.
% The details of the MLLS are as follows. 
% \iffalse
Each round of the MLS algorithm involves three levels, with each level conducting searches of varying breadth and depth based on different exchange operations.  
Initially, the first level local search focuses on the two partitions with the largest difference in weight and explores all possible exchange operations to execute the best one, which may 
%Not only does this enhance the solution, but it also results in a more balanced partitioning scheme. 
not only enhance the solution but also result in a more balanced partitioning scheme. 
If the solution is successfully optimized, the same level is continued and repeated until no further solution optimization is possible, at which point the second level local search is employed. %The second level simplifies the exchange operations due to the 
The second level considers more partitions but simpler exchange operations, exploring the possible exchanges and stopping when an improvement operator is found.                                                                                                      
%increased number of partitions considered, which 
%. While 
% further attempts to improve the solution and also opens up possibilities for finding optimized approaches in the next round of first-level. 
If both of the above methods fail to improve the solution, indicating that the current solution is challenging to optimize through small-scale vertex exchanges. Consequently, the reorganization of the two partitions with the largest weight difference is designed as the third level. If the modification improves the solution, it is retained, and the search continues into the next round, starting from the first level. Otherwise, the local search terminates. %concluded.
%Initially, the first-level local search attempts to exchange vertices in $U_{max}$ with those in $U_{min}$, aiming to achieve a more balanced distribution of weight. If this strategy no longer yields further optimization of the objective, the second-level local search is employed. This level attempts exchange vertices in $U_{max}$ with all other vertices, thereby exploring a broader range of solution space.
% potential improvements.
% 如果这两种方法都失败，则尝试使用2-way Karmarkar-Karp算法对$U_{max}$和$U_{min}$中的节点进行重新分配
%If both of these methods fail, MLLS then attempts to use the 2-way Karmarkar-Karp algorithm to redistribute the vertices within $U_{max}$ and $U_{min}$ as the third-level strategy.
% \fi

% Each round of the MLLS involves three levels, where the first two levels try to improve the 
Let $U_{max}$ and $U_{min}$ respectively be the partitions with the largest and smallest weights. The specific operations of each level of the MLS algorithm are as follows.

\begin{itemize}
\item{\textbf{First Level:} Select vertices from $U_{max}$ and $U_{min}$, respectively, for exchange. First, sort vertices in $U_{min}$ by their weights. Attempt to swap a vertex from $U_{max}$ with several consecutive vertices in $U_{min}$. Execute the best operator, specifically the one that minimizes the weight difference between the two partitions. The enumeration range for vertices in $U_{min}$ can be pruned based on monotonicity.}
\item{\textbf{Second Level:} Attempt to swap a vertex from $U_{max}$ with a vertex from another partition. Traverse the partitions in ascending order of their weights and execute the first operation found that can optimize the solution quality.}
\item{\textbf{Third Level:} Use the 2-way Karmarkar-Karp algorithm to redistribute the vertices within $U_{max}$ and $U_{min}$}
\end{itemize}

\subsubsection{Elite Strategy for Crossover and Population Update}
\label{subsec:elite}

% 一般的hybrid evolution算法每次随机挑选两个种群中的解作为父代进行杂交。这样虽然能保证算法的多样性，但效率不够高。我们的方法对父代的选择进行改进，每次都将种群中最优的个体作为其中一个父代，另一个父代仍采用随机的方式。我们称这种策略为精英策略，最优的个体为精英解$sol_{elite}$。令随机生成的父代以及产生的子代分别为$sol_{random}$和$sol_{new}$
% 随机选择的方式虽然能保证多样性，但是后代不能充分继承优秀个体的信息，导致高效的探索解空间
Crossover plays an essential role in genetic algorithms, and different algorithms have their respective methods for selecting the parents. Although random selection, as the most commonly used method, ensures diversity, the generated offspring do not inherit enough information from the elite (best) individuals, leading the exploration of the solution space to be inefficient. To address this issue, we introduce the elite strategy in our genetic algorithm. The elite strategy consistently selects the elite individual as one of the parents in each crossover, attempting to integrate the features of other individuals into the elite individual in each, thereby continuously optimizing the elite. This approach ensures a high quality of exploration in the solution space while maintaining a degree of diversity.

% In typical hybrid evolutionary algorithms, two solutions from the population are randomly selected as parents for crossover. While this approach ensures diversity in the algorithm, it is not particularly efficient. In this work, we consistently choose the best individual in the population as one of the parents, and another parent is still selected randomly. This approach is referred to as the elite strategy, where the best-performing individual is designated as the elite solution. 

% 换种方式取理解，每次杂交都是通过在精英解中加入种群中其他解的特征，尝试对精英解继续改进。使得每轮迭代更有目的性，在大幅提升效率的同时也保证了一定的多样性。 
% By this way, each crossover operator in our approach involves incorporating features from other solutions in the population into the elite solution, with the aim of further improving the elite solution. This method ensures that each iteration is more targeted, significantly enhancing efficiency while also maintaining a degree of diversity. 
% By focusing on refining the elite solution with attributes from various other solutions, we strike a balance between intensifying the search around the most promising areas and exploring new possibilities, thereby optimizing the algorithm's performance.

% 为了配合精英策略，子代更新种群的策略也和常规的方法有区别。在对子代进行改进后，若当生成子代优于精英解时，该子代成为新的精英个体，替换掉种群中原先的精英解。若子代的目标函数值等于其中一个父代，我们认为$sol_{new}$和$sol_{random}$没有合适的特征来改进精英解，并且多样性也不够。所以不将$sol_{new}$加入种群，并且将$sol_{random}$从种群中移除，来加速迭代过程。
% 种群更新的mechanism同样也体现了elite strategy
The mechanism for population update also reflects the elite strategy. After optimizing the offspring $\mathcal{P}_{child}$ using the MLS algorithm, the strategy for population update via the $\mathcal{P}_{child}$ is shown in \ralg{UP}. If it is better than the elite, this offspring becomes the new elite individual, replacing the previous one (lines 4 to 5). If the offspring's objective value is equal to one of its parents, it indicates that neither $\mathcal{P}_{child}$ nor $\mathcal{P}_{mate}$ possess suitable features to improve the elite. Therefore, $\mathcal{P}_{mate}$ is directly removed to accelerate the iteration process (line 8). Otherwise, $\mathcal{P}_{mate}$ is replaced with the offspring (lines 9 to 10). This update method accelerates the search process by focusing on the most promising solutions.

\begin{algorithm}[t]
	\caption{Update Population}
	\label{alg:UP}	
	\begin{algorithmic}[1]
		\Input the population $Pops$, the mate individual $\mathcal{P}_{mate}$ and the offspring $\mathcal{P}_{child}$ during this round crossover
		\Output the updated population $Pops$
		\Function{UpdatePopulation}{$Pops$, $\mathcal{P}_{mate}$, $\mathcal{P}_{child}$}
        \State $\mathcal{P}_{elite} \gets $ best individual in $Pops$
        \If {$f(\mathcal{P}_{child}) < f(\mathcal{P}_{elite})$}
                \State $Pops \gets Pops \backslash \{\mathcal{P}_{elite}\}$
                \State $Pops \gets Pops \cup \{\mathcal{P}_{child}\}$ 
        \Else
            \If {$f(\mathcal{P}_{child})$ = $f(\mathcal{P}_{elite})$ \OR $f(\mathcal{P}_{child})$ = $f(\mathcal{P}_{mate})$}
                \State $Pops \gets Pops \backslash \{\mathcal{P}_{mate}\}$
            \Else
                \State $Pops \gets Pops \backslash \{\mathcal{P}_{mate}\}$ 
                \State $Pops \gets Pops \cup \{\mathcal{P}_{child}\}$ 
            \EndIf
        \EndIf
        \State \Return $Pops$
		\EndFunction
	\end{algorithmic}                                      
\end{algorithm} 

% 算法的结束条件为，种群中只剩下一个解或者在一定轮数的迭代内没有对精英解进行改进。经过测试，在算法中，这个轮数设为种群数量的4倍。
% The HEE algorithm concludes either when only one solution is left in the population or if the elite solution has no improvement over a set number of iterations. After testing, we set this number to be four times the population size.

\subsection{Graph Modification Strategy}
\label{subsec:ER}
% 算法的第三阶段通过对二分图中边的禁忌来改变匹配结果，从而影响解的质量
% 在得到这一轮的解后，FIMP-HGA通过对边进行禁忌和解禁来改变下一轮的匹配结果。不同的匹配方案意味着对不同解空间的探索，为改进当前最优解提供可能。
After obtaining the solution for the current round, FIMP-HGA modifies the matching results for the next round by forbidding and releasing edges. Different matching schemes imply the exploration of diverse solution spaces, providing possibilities for improving the optimal solution.
% In the third phase of the algorithm, we adjust the bipartite graph by tabuing edges to change the matching, thereby affecting the quality of the solution. 
% 影响边的禁忌情况的策略主要有以下两种
The strategies influencing the tabu status of edges are as follows.

\subsubsection{Tabu Strategy in Each Iteration}

% 根据每次迭代的分区结果，对权重最大的分区中权重最大的顶点的匹配边进行禁忌。这种方式能保证下一轮迭代的匹配方案不同，且有可能匹配结果能让分区的权重更均衡。
Based on the partition results of each iteration, the matching edge of the vertex with the largest weight in the partition $U_{max}$ is tabued. This approach ensures that the matching scheme in the next iteration is different, and it is possible that the resulting match could lead to a more balanced distribution of weights across the partitions.

% 禁忌的轮次根据顶点数量和图的密度设定，经过测试具体设置如下：
The number of tabu rounds is determined based on the number of vertices and the density of the graph. The specific settings within FIMP-HGA are as follows:

\begin{equation}
\label{eqn:Tabu}
    TabuStep = min(0.2 \times |U|, 0.1  \times |E|).
\end{equation}

\subsubsection{Edge Recovery Strategy}

% 当由于禁忌边导致当前解离最优解差距过大，说明搜索的解空间优度不够，通过一次性解禁多条边来提升解的优度。我们称这个过程为Edge Recovery
When the current solution deviates significantly from the optimal solution due to the tabu edges, it indicates that the quality of the solution space being searched is insufficient. To escape this situation, the edge recovery strategy is employed, which involves releasing multiple edges at once to break out of the current search space.

% 定义算例的理论下界为原图匹配结果$\Pi_G$中匹配边的权重和除分区数，即
The theoretical lower bound of an instance is defined as the total weights of the edges in the matching result $\Pi_G$ of the original graph, divided by the number of partitions, denoted as follows:

\begin{equation}
\label{eqn:LB}
    LowerBound = \frac{\sum_{e_{uv} \in \Pi_{G}} w(e_{uv})}{m}.
\end{equation}

% 当当前解$\Pi,mathcal{P}$和最优解$\Pi_{ans}, mathcal{P}_{ans}$满足下面式子时，进行Edge Recovery
Edge recovery is performed when the current solution $\Pi, \mathcal{P}$ and the optimal solution $\Pi_{ans}, \mathcal{P}_{ans}$ meet the following condition: 

\begin{equation}
\label{eqn:ER}
    \frac{f(\Pi, \mathcal{P})-f(\Pi_{ans}, \mathcal{P}_{ans})}{f(\Pi, \mathcal{P}) - LowerBound} > 0.9.
\end{equation}

% Edge Recovery的策略是对于当前禁忌的每条边以$50\%$的概率马上解禁
The strategy for edge recovery is to release all tabu edges with a $50\%$ probability immediately.
% 当迭代中多次执行Edge Recovery操作证明解空间中能搜到的质量较好的解都已经搜索过，直接结束迭代
When the number of executed edge recovery operations reaches a certain count, indicating that FIMP-HGA has substantially explored the accessible high-quality solution space, the algorithm concludes directly to prevent unnecessary searches.

\section{Experimental Results}
\label{subsec:exp}
% 这一章首先介绍新数据集得生成方式，然后用该数据集对本文提出得算法进行测试。
This section first introduces the method for generating new benchmarks that are compatible with the existing ones and then employs the new benchmarks to evaluate our proposed FIMP-HGA \footnote{The code of our algorithm and the generated benchmarks will be available online upon acceptance.}. 

\subsection{Instance Classification and Generation}
\label{subsec:DS}
% 对问题特征进行分析以及参考了之前论文中数据集的生成方式。提取问题的两个关键特征：一致性和图密度。再对不同特征的数据根据节点数量，分区数量以及区域内顶点的上限生成数据集。
%After analyzing the characteristics of the problem and referencing benchmark generation methods from previous papers, we identified two key features of the PMMWM problem: consistency and bipartite graph density. We then generated datasets with varying characteristics based on the number of vertices, the number of partitions, and the upper limit of vertices within each region, taking these key features into account.
%Since $n_1\neq n_2$ can be reduced to the case of $n_1 = n_2$ by adding edges with zero weight, we define $n: = n_1 = n_2$ in this paper.
% 接下来分别对这些特征进行详细描述。
% Next, we will provide detailed descriptions of each of these features.

% 生成算例集时，需要把特征不同的算例都考虑到。对于PMMWM问题，除了算例规模，分区数量和分区顶点数量上限外，还需要考虑二分图的特征。我们将不同的二分图通过两个特征进行分类，分别是：一致性和密度。
It is important for benchmarks to include instances with diverse features. For instances of the PMMWM problem, alongside factors such as the number of vertices, the number of partitions, and the capacity of each partition, the structure of the bipartite graph is also important. We categorize different bipartite graphs based on two features: consistency and density. Firstly, we provide a comprehensive description of these two features. Subsequently, we analyze the relationship between the generated instances and the previous ones. Since $n_1 \neq n_2$ can be reduced to the case of $n_1 = n_2$ by adding edges with zero weight, we define $n: = n_1 = n_2$ in this paper. 
% 接下来我们首先提供这两个特征的细节描述，然后分析生成的算例与之前的关系。

\subsubsection{Consistency}
\label{subsec:con}

% 一致性表示二分图中vertex sets $U=\{u_1,...,u_n\}$中每个顶点连向$V=\{v_1,...,v_n\}$中顶点的边权大小关系是否一致。若$U$中每个顶点连接的权值最大的边都连向$v_1$，则称$v_1$满足一致性。若$U$中每个顶点连接的权值最大的前三边边都分别连向$v_1,v_2,v_3$，则称这个图中$V$集前三个顶点满足一致性。
% 一致性用于衡量在集合$U=\{u_1,u_2,...,u_n\}$中的顶点对于集合$V=\{v_1,v_2,...,v_n\}$中某个顶点的匹配优先级是否一致。
%Consistency refers to whether the relationships between edge weights connecting each vertex in $U = {u_1, ..., u_n}$  to the vertices in $V = {v_1, ..., v_n}$ are uniform. 

This feature is used to measure whether the vertices in $V = \{v_1, v_2, ..., v_n\}$ have a consistent matching priority for the vertex in $U = \{u_1, u_2, ..., u_n\}$. 
For example, if the edge with the maximum weight from each vertex in $V$ consistently connects to $u_1$ in $U$, then $u_1$ is said to satisfy consistency. Similarly, if the top three edges with the highest weights from each vertex in $V$ respectively connect to $u_1, u_2,$ and $u_3$ in $U$, then these three vertices in $U$ are considered to satisfy consistency in this graph.

% 规定一致性的权重$Con$为$0-100$范围内的值，表示图中$V$集前百分之$Con$的顶点满足一致性。
The consistency weight, denoted as $Con$, is defined within the range of $0-100$. It represents the percentage of vertices in set $U$ that satisfy consistency in the graph. 
% Specifically, the top $Con\%$ of vertices in the $V$ set satisfy the consistency criterion.
% 一致性越高，算法matching阶段的决策就越困难。同时在partition阶段顶点的权重差异性更大，更考验partition阶段算法的优度。
% The higher the consistency, the more challenging the decision-making becomes during the matching phase of the algorithm. Simultaneously, in the partition phase, the greater disparity in vertex weights tests the optimality of the algorithm used in this stage. 
Higher consistency not only increases the complexity of decision-making in the match stage but also leads to larger differences in the vertex weights, which challenges the effectiveness of the partition stage.
% 具体的算例生成方式如下
The method for generating instances with specific consistency features is as follows:

We first randomly generate $n^2$ rational numbers on the interval $[1, 1000]$ and a complete bipartite graph $G(U,V,E)$ with vertex sets $U=\{u_1,...,u_n\}$ and $V=\{v_1,...,v_n\}$. We sort the rational numbers in non-decreasing order and store them in list $L$. Then we repeat the following operations on each vertex $v_i \in V$ from $v_1$ to $v_n$, to assign $n$ rational numbers in $L$ to the $n$ edges incident with $v_i$ in $E$. 
For $v_i \in V$, we assign the first $\lfloor n*Con\%\rfloor$ elements of $L$ to the edges $(u_1 ,v_i)$, $(u_2,v_i)$, $...$, $(u_{\lfloor n*Con\%\rfloor},v_i)$ and remove these elements from $L$. 
For the remaining edges incident with $v_i$, we randomly assign the remaining elements in $L$ in turn and remove the assigned elements from $L$.

\subsubsection{Density}

% 上一章生成的算例均为完全图，但二分图的密度也是反应图结构的重要特征。为了充分评估算法，在上一章算例的基础上随机删除一些边来控制二分图的密度。
Following the instance generation method of the previous section, the resulting graphs are all complete. To fully evaluate the algorithm, some edges are randomly removed from those instances, changing the graph's density while preserving its consistency features.
% 规定二分图的密度为权重$Den$为$0-100$范围内的值，表示该算例中边数为完全二分图的百分之$Den$。
The density of the bipartite graph is defined by weight, $Den$, which falls within the range of $0-100$. This value represents that the number of edges in the given example is equal to $Den\%$ of what it would be in a complete bipartite graph. 
% 由于要保证删边后算例存在合法解，因此保留所有$u_i$连向$v_i$的边。在剩余的边中随机选择$n^2 * (1-Den\%)$条边删除。这个方法在保留一致性特征的前提下改变图的密度
To ensure that a valid solution exists after edge deletion, all edges connecting $u_i$ to $v_i$ are retained. From the remaining edges, we randomly select and delete $n^2 * (1-Den\%)$ edges.

\subsubsection{Relationship with Previous Instances}
% 我们生成Instance的方式兼容MPLS的测试集。对于所有算例，根据特征Consistency和Density的权重进行分类，如一致性权重为$70$，二分图密度权重为$80$的数据可以归类为Con70Den80。对于之前测试MPLS的每个benchmarks，BPS70,BPS80,RAND, SPARSE70,SPARSE80, 都对应我们数据集中其中一种，分别是BPS70,BPS80,RAND,SPARSE70和SPASE80。在我们的算例生成方式下，BPS70和BPS80分别对应Con70Den100和Con80Den100；RAND对应Con0Den100；SPARSE70和SPARSE80对应Con0Den70和Con0Den80。
Our method of generating instances is compatible with the previous approaches \cite{a1}.
The instances are classified based on the weights of the consistency and density features. For example, the instance with a consistency weight of $70$ and a bipartite graph density weight of $80$ can be categorized as Con70Den80. For each of the benchmarks previously used to test MP$_\text{LS}$ — BPS70, BPS80, RAND, SPARSE70 and SPARSE80, there is a corresponding type in our instances. The BPS70 and BPS80 correspond to Con70Den100 and Con80Den100, respectively; the RAND corresponds to Con0Den100; the SPARSE70 and SPARSE80 correspond to Con0Den70 and Con0Den80, respectively. 
Therefore, our generated instances represent an expansion and refinement of the previous dataset.

% 为了全面的评估算法在不同场景下的效率和性能，我们生成了$Con \in \{0, 25, 50, 75, 100\}$，$Den \in \{25, 50, 75, 100\}$共20种特征不同的数据作为算例集。
% To comprehensively evaluate the algorithm's efficiency and performance across different scenarios, we generated a benchmark with 20 distinct feature combinations, encompassing $Con \in \{0, 25, 50, 75, 100\}$ and $Den \in \{25, 50, 75, 100\}$.

\begin{comment}

\subsubsection{Benchmarks}

For each category, we generate instances with different parameter settings. For generating small instances, we vary the size of the bipartite graphs from 30 to 190, i.e., $n=30,50,70, ...,190$. For generating large instances, we vary the size of the bipartite graphs from 200 to 900, i.e., $n = 200,300,400,...,900$. 
% 如果a-b>12我们将m设置为1，2，3或4，否则若a-b>8我们将设置为1，2，3.若a-b<=0我们将m设置为1，2
If $\lfloor0.125n\rfloor - 2>12$, we set m to  2, $\lfloor0.04n\rfloor$, $\lfloor0.08n\rfloor$  or $\lfloor0.125n\rfloor$. Otherwise, if $\lfloor0.125n\rfloor - 2>8$, we set it to 2, $\lfloor0.04n\rfloor$ or $\lfloor0.125n\rfloor$. If $\lfloor0.125n\rfloor - 2 \leq 12$, we set m to 2, $\lfloor0.125n\rfloor$. 
We set the maximum number of vertices in each partition, $\bar{u}$ to $\lceil\dfrac{n}{m}\rceil$ or $n$. 
% 总共生成1160个小算例，1280个大算例。
% 最终Benchmark中共有1160个小算例，1280个大算例。
A total of 1160 small instances and 1280 large instances in our benchmark.

\end{comment}

\subsection{New Benchmarks}
% To comprehensively evaluate the algorithm's efficiency and performance across different scenarios, we generated three benchmarks with 20 distinct feature combinations, encompassing $Con \in \{0, 25, 50, 75, 100\}$ and $Den \in \{25, 50, 75, 100\}$

% benchmarks分为如下两类
% 第一个benchmark为第一类。用于全面的评估算法在不同场景下的效率和性能。通过大量特征不同的数据进行综合评估。
% 后面的3个benchmarks为第二类。用于测试二分图的密度和一致性以及分区数量对结果和算法的影响。通过控制变量法，固定其他参数的不变，只改变需要测试的数值

We employ our instance generation method described in the previous section to construct four benchmarks for a comprehensive evaluation. The benchmarks are divided into the following two categories:

\begin{itemize}
    \item{The first benchmark belongs to the first category. It is used to comprehensively evaluate the algorithm's efficiency and performance across different scenarios. This is achieved through an overall evaluation using a large number of instances with different features. All benchmarks used for testing the MP$_{\text{LS}}$ algorithm are reflected in this benchmark.}
    \item{The subsequent three benchmarks belong to the second category. They are used to assess how the density and consistency of bipartite graphs, as well as the number of partitions, influence the outcomes and the performance of the algorithms. The generation of benchmarks is achieved by using the method of controlling variables, keeping other parameters constant, and only changing the values that need to be tested.}
\end{itemize}

The details of these four benchmarks are as follows.

\subsubsection{Benchmark-All}

% 用于全面的评估算法在不同场景下的效率和性能
This benchmark has 20 distinct feature combinations, encompassing $Con \in \{0, 25, 50, 75, 100\}$ and $Den \in \{25, 50, 75, 100\}$. For each category, we generate instances with different parameter settings. The size of the bipartite graphs from 50 to 500, i.e., $n=50,100,150, ...,500$. If $\lfloor0.125n\rfloor - 2>12$, we set $m$ to 2, $\lfloor0.04n\rfloor$, $\lfloor0.08n\rfloor$, or $\lfloor0.125n\rfloor$. Otherwise, if $\lfloor0.125n\rfloor - 2>8$, we set it to 2, $\lfloor0.04n\rfloor$, or $\lfloor0.125n\rfloor$. If $\lfloor0.125n\rfloor - 2 \leq 12$, we set $m$ to 2 or $\lfloor0.125n\rfloor$. The maximum number of vertices in each partition $\bar{u}$ is set to $\lceil\dfrac{n}{m}\rceil$, or $n$. We generate one instance for each parameter setting, so there are a total of 1480 instances in this benchmark. 
% 

% To comprehensively evaluate the algorithm's efficiency and performance across different scenarios, we generated a benchmark with 20 distinct feature combinations, encompassing $Con \in \{0, 25, 50, 75, 100\}$ and $Den \in \{25, 50, 75, 100\}$.

% For each category, we generate instances with different parameter settings. For generating small instances, we vary the size of the bipartite graphs from 30 to 190, i.e., $n=30,50,70, ...,190$. For generating large instances, we vary the size of the bipartite graphs from 200 to 900, i.e., $n = 200,300,400,...,900$. 
% 如果a-b>12我们将m设置为1，2，3或4，否则若a-b>8我们将设置为1，2，3.若a-b<=0我们将m设置为1，2
% If $\lfloor0.125n\rfloor - 2>12$, we set m to  2, $\lfloor0.04n\rfloor$, $\lfloor0.08n\rfloor$  or $\lfloor0.125n\rfloor$. Otherwise, if $\lfloor0.125n\rfloor - 2>8$, we set it to 2, $\lfloor0.04n\rfloor$ or $\lfloor0.125n\rfloor$. If $\lfloor0.125n\rfloor - 2 \leq 12$, we set m to 2, $\lfloor0.125n\rfloor$. 
% We set the maximum number of vertices in each partition, $\bar{u}$ to $\lceil\dfrac{n}{m}\rceil$ or $n$. 
% 总共生成1160个小算例，1280个大算例。
% 最终Benchmark中共有1160个小算例，1280个大算例。
% A total of 1160 small instances and 1280 large instances in our benchmark.

\subsubsection{Benchmark-Den}

% 用于测试二分图的密度对结果和算法的影响。通过控制变量法，固定其他参数的不变，只改变$Den$的数值。具体的，令$n=300,m=20,\bar{u}=18,Den = 20,30,40...100$。对于$Den$的每种不同取值生成10个instance,共90个instance。
%Used to test the impact of the density of bipartite graphs on the results and the algorithm. By employing the method of controlling variables, all other parameters are held constant, with only the value of $Den$ being altered. Specifically, 

We set $n=300, m=20, \bar{u}=18, Con=75, Den=20, 30, 40, ..., 100$. For each different value of $Den$, 10 instances are generated, resulting in a total of 90 instances.

\subsubsection{Benchmark-Con}

We set $n=300, m=20, \bar{u}=18, Den=75, Con=0, 10, 20, ..., 100$. For each different value of $Con$, 10 instances are generated, resulting in a total of 110 instances.

\subsubsection{Benchmark-M}

We set $n=300, \bar{u}=300, Den=75, Con=75, m=2, 7, 12, ..., 32$. For each different value of $m$, 10 instances are generated, resulting in a total of 70 instances.

\subsection{Experimental Setup}

\subsubsection{Experimental Settings}

All the experiments were performed on an AMD Ryzen 5 3600 6-Core Processor at 3.60 GHz and 16GB of RAM, running Windows10 64-bit. All the tested algorithms are implemented in C++ with ``-O2'' option.

\subsubsection{Comparison Program}

% 我们使用当前解决PMMWM问题效果最好的MP_$\text{LS}$算法进行对比。我们将其复现后通过原论文确定正确性，然后在相同的实验环境下进行测试。
We compare FIMP-HGA with the MP$_\text{LS}$ algorithm, which is the state-of-the-art algorithm for solving the PMMWM problem. We implemented the MP$_\text{LS}$ algorithm as described by \citet{a1} and verified its correctness. Subsequently, we ran MP$_\text{LS}$ and the five versions of our proposed algorithm, comprehensively evaluating FIMP-HGA through the results on the new benchmarks. The names and introductions of the algorithms are as follows.

\begin{itemize}
\item{\textbf{MP$_\text{LS}$:} The state-of-art algorithm for PMMWM problem.}
\item{\textbf{FIMP-LS}: Replace the KM algorithm used in the match stage of each iteration in MP$_\text{LS}$ with the KM-M algorithm.}
\item{\textbf{FIMP-MLS-GD:} Based on FIMP-LS, change the local search algorithm used in the partition phase to the MLS algorithm proposed in this paper. Initial solutions are generated using the GD method.}
\item{\textbf{FIMP-MLS:} Similar to FIMP-MLS-GD, change the initial solution to be generated using the KK method.}
\item{\textbf{FIMP-HGA-GD:} Based on FIMP-MLS-GD, use HGA for the partition stage. For the initial population, use the GD and RGD methods. Furthermore, the graph modification strategy outlined in this paper is implemented..}
\item{\textbf{FIMP-HGA:} Based on FIMP-MLS-GD, add the KK method to the initial population of HGA. This is the final algorithm proposed in our paper.}
\end{itemize}

% we conducted comparative tests with FIMP-HGA under the same experimental conditions.
%After replicating it, we verify its correctness through the original paper and then conduct tests in the same experimental environment.

% 我们将MP$_\text{LS}$和5个版本我们提出的算法在算例集上运行，通过运行结果全面评估FIMP-GA。命名与算法构成如下

% \subsubsection{Params Setting}

% FMP$_{\text{HEE}}$中涉及一些常数的设定，其对应含义和值在\rtbl{PARAM}中展示
% 在实验中FIMP-HGA的静态参数设置如表\rtbl{PARAM}所示
In the experiment, the static parameter settings of the modules we propose are listed in Table \rtbl{PARAM}.

\begin{table}[t]
\centering
\caption{Parameter settings.}
\label{tbl:PARAM}
\scalebox{1.0}{
\begin{tabular}{@{}llr@{}}
\toprule
Parameter       & Description                               & Value \\ \midrule
$T$           & maximum unimproved iterations in FIMP-HGA             & 20    \\
$T_{ER}$       & edge recovery execution limits in FIMP-HGA                     & 8     \\
$T_{HGA}$        & maximum unimproved iterations in HGA       & 20  \\
$|Pops|$         & population size in HGA                           & 5   \\ \bottomrule
\end{tabular}}
\end{table}

% \subsubsection{Scheme Naming}

% 为了测试算法不同模块对效率和优度的优化效果。
% 我们通过各模块的使用情况，生成了5种
% To more thoroughly test FIMP-HGA, we conducted tests on several different schemes. First, we assigned names to these various schemes.
% 目前解决PMMWM最好的算法
% 将MP$\text{LS}$每轮迭代中匹配阶段使用的KM算法改为KM-M算法
% 在FMP\_LS基础上，将分区阶段使用的启发式算法该文本文中提到的three-level local search。初始解使用贪心算法生成。
% 与FMP\_ILS-GD一样，将初始解改为使用Karmarkar-Karp算法生成。
% 在FMP\_ILS-GD基础上，在分区匹配阶段，使用HEE算法进行搜索。初始种群时，使用GD和RGD方法。并采用本文中对图进行的更改策略。
% 在FMP\_ILS-GD基础上，对HEE算法初始种群时，加入KK方法。即我们最终提出的算法。

\subsubsection{Evaluation Metrics}

The metrics used to evaluate the performance of the algorithm $Alg$ for a given instance include the solution quality $SQ_{Alg}$ and the runtime $RT_{Alg}$. Each algorithm is run five times with different seeds on an instance, and the average value is taken to reduce the randomness. 

For different algorithms we propose, their effectiveness is evaluated based on the metric improvement over  MP$_\text{LS}$ across different instances. 
The effectiveness of the algorithm in optimizing the solution quality is benchmarked against the theoretical lower bound (\reqn{LB}) of the test instances.
% The result of scheme $Alg$ is evaluated with reference to the theoretical lower bound (\reqn{LB}) of the tested instance. 
The improved objective value $V_{Opt}$ and improvement ratio $R_{Opt}$ of $Alg$ compared to MP$_\text{LS}$ is calculated using \reqn{res}.
% The performance of result is evaluated with reference to the theoretical lower bound (\reqn{LB}) of the instance. The extent of optimization achieved by scheme $S$ compared to MP$_\text{LS}$ is assessed using \reqn{res}. 
If the result of MP$_\text{LS}$ already reaches the theoretical lower bound, then $R_{Opt}$ is set to 0. 

\begin{equation}
\begin{aligned}
\label{eqn:res}
    V_{Opt} &= SQ_{\text{MP}_{\text{LS}}} - SQ_{Alg}, \\
    R_{Opt} &= \frac{V_{Opt}}{SQ_{\text{MP}_{\text{LS}}}-LowerBound} \times 100.
\end{aligned}
\end{equation}

% 运行时间通过\reqn{time}来评估方案$S$相较于MP$\text{LS}$的速度提升。
Similarly, the runtime improvement of $Alg$ over MP$\text{LS}$ is evaluated using \reqn{time}.

\begin{equation}
    \label{eqn:time}
    T_{Opt} = \frac{RT_{\text{MP}_{\text{LS}}} - RT_{Alg}}{RT_{\text{MP}_{\text{LS}}}} \times 100.
\end{equation}

% 接下来对提出的算法不同方案均用$OptRes$和$OptTime$进行评估。
% When $ROpt$ or $TOpt$ is a negative number, it similarly indicates the extent to which the result or running time of scheme $S$ is worse compared to MP$_\text{LS}$.
All the developed algorithms are evaluated using $V_{Opt}$, $R_{Opt}$, and $T_{Opt}$.

% 时间用倍数还是百分比

% \subsection{Comparison with State-of-the-art Algorithms}
\subsection{Results and Analyses}

% 我们首先在Benchmark-All上对所有方案进行测试，分析不同策略对算法效果的影响。然后在剩余benchmark上进一步实验并分析。
We first test all algorithms on Benchmark-All to analyze the impact of different strategies on the algorithm performance. Then, we conduct further experiments and analysis on the remaining benchmarks.

\subsubsection{Comprehensive Evaluation}

% 将不同方案在Benchmark-All上的测试结果根据二分图节点数量进行分类后分别取平均值，得到\rtbl{all}。
After categorizing the test results of different algorithms on Benchmark-All based on the number of vertices in the bipartite graph, we calculate the average values separately, resulting in \rtbl{all}. 
% 并用\rfig{NRes}和\rfig{NTime}分别展示不同方案下ROpt和TOpt随n的变化趋势。
We use \rfig{NRes} and \rfig{NTime} to respectively display the trends of $R_{Opt}$ and $T_{Opt}$ with respect to $n$ for different algorithms.

\begin{table*}[t]
\centering
\caption{Improvement of all algorithms over MP$_\text{LS}$ on the Benchmark-All.}
\begin{threeparttable}
\label{tbl:all}
\sisetup{
  mode=text, 
  detect-all
}
\begin{tabular}{@{}cSSSSSSSSSS@{}}
\toprule
\multirow{2}{*}{$n$}& \multicolumn{2}{c}{FIMP-LS} & \multicolumn{2}{c}{FIMP-MLS-GD} & \multicolumn{2}{c}{FIMP-MLS} & \multicolumn{2}{c}{FIMP-HGA-GD} & \multicolumn{2}{c}{FIMP-HGA} \\
\cmidrule(r){2-3} \cmidrule(r){4-5} \cmidrule(r){6-7} \cmidrule(r){8-9} \cmidrule(r){10-11}
& {$R_{Opt}$} & {$T_{Opt}$} & {$R_{Opt}$} & {$T_{Opt}$} & {$R_{Opt}$} & {$T_{Opt}$} & {$R_{Opt}$} & {$T_{Opt}$} & {$R_{Opt}$} & {$T_{Opt}$} \\
\midrule
50  & 0.00 & 89.69 & 21.92 & 92.93 & 55.92 & 91.48 & 73.18 & 76.52 & 91.06 & 71.31 \\
100 & 0.11 & 92.97 & 39.34 & 94.48 & 63.31 & 94.22 & 76.29 & 80.69 & 94.50 & 81.94 \\
150 & 0.35 & 93.81 & 51.73 & 95.22 & 65.70 & 94.87 & 79.70 & 81.99 & 95.00 & 85.63 \\
200 & -1.04 & 94.47 & 61.58 & 95.52 & 70.91 & 95.29 & 80.81 & 84.67 & 94.69 & 87.64 \\
250 & -0.47 & 94.96 & 66.58 & 95.62 & 73.26 & 95.53 & 83.14 & 86.54 & 95.19 & 90.36 \\
300 & -0.54 & 95.16 & 73.12 & 95.82 & 78.65 & 95.74 & 83.30 & 88.37 & 95.59 & 91.59 \\
350 & -0.31 & 95.23 & 78.87 & 95.80 & 84.69 & 95.72 & 87.81 & 89.16 & 95.81 & 92.13 \\
400 &  0.95 & 95.19 & 78.77 & 95.80 & 80.69 & 95.70 & 86.40 & 89.79 & 93.93 & 92.62 \\
450 & -1.03 & 95.53 & 78.80 & 96.14 & 83.97 & 96.06 & 87.49 & 91.03 & 94.56 & 93.72 \\
500 & -1.61 & 95.56 & 82.09 & 95.99 & 84.88 & 95.99 & 88.04 & 92.14 & 93.78 & 93.91 \\
\bottomrule
\end{tabular}
\begin{tablenotes}
\item $R_{Opt}$ and $T_{Opt}$ can be understood as the percentage of improvement ratio.
\end{tablenotes}
\end{threeparttable}
\end{table*}

% 我们提出的算法在Benchmark-All上ROpt随n的变化趋势。对于不同方案随n增大ROpt均增加

\begin{figure}[t]
  \centering
  \begin{subfigure}{\linewidth}
    \centering
    \includegraphics[width=\linewidth]{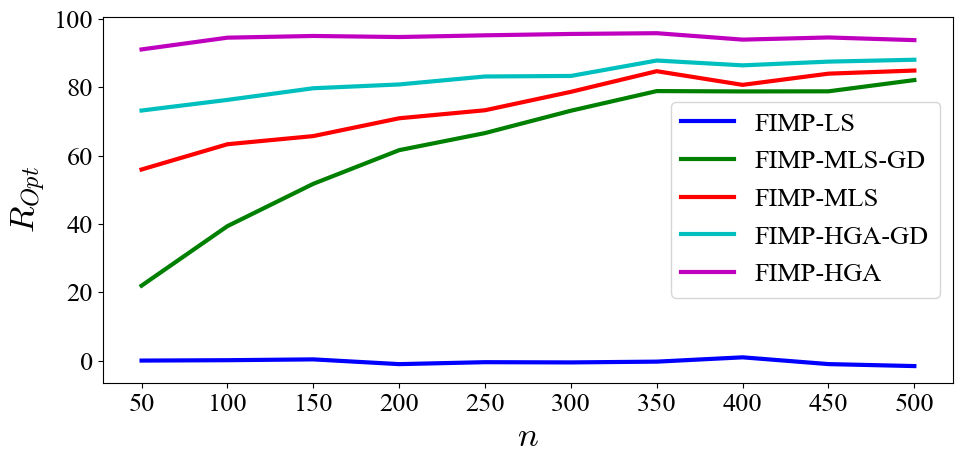}
    \caption{}
    \label{fig:NRes}
  \end{subfigure}\\ 
  % 下面的图片
  \begin{subfigure}{\linewidth}
    \centering
    \includegraphics[width=\linewidth]{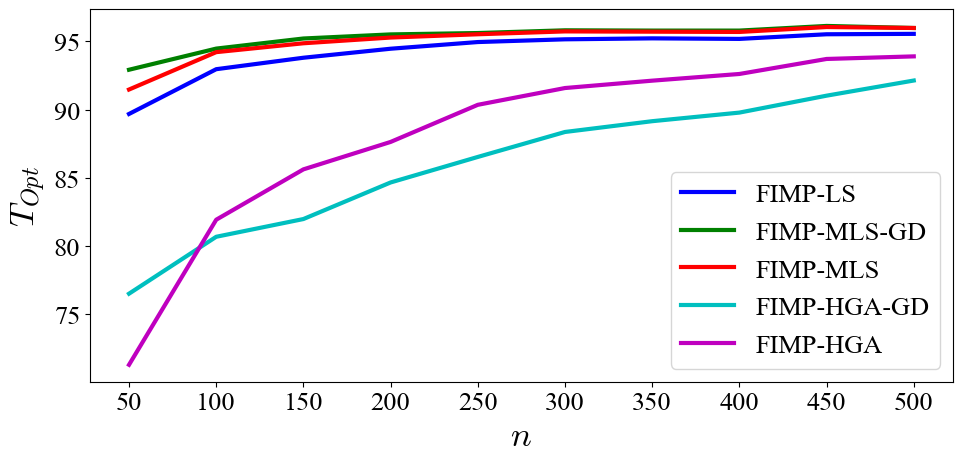} 
    \caption{}
    \label{fig:NTime}
  \end{subfigure}
  \caption{Trend of $R_{Opt}$ (a) and $T_{Opt}$ (b) for our proposed algorithmscga on Benchmark-All with $n = 50$ to $500$. $R_{Opt}$ and $T_{Opt}$ all increase with $n$ for all algorithm.}
  \label{fig:NAll}
\end{figure}

\begin{table*}[t]
\caption{Improvement of three algorithms over MP$_\text{LS}$ on Benchmark-Den and their concrete performance.}
\label{tbl:den}
\centering
\sisetup{
  mode=text, 
  detect-all
}
\begin{tabular}{@{}cccSccSccccccccc@{}}
\toprule
\multirow{2}{*}{$Den$} & \multirow{2}{*}{LB} & \multicolumn{2}{c}{MP$_\text{LS}$} & \multicolumn{4}{c}{FIMP-LS} & \multicolumn{4}{c}{FIMP-MLS} & \multicolumn{4}{c}{FIMP-HGA} \\
\cmidrule(r){3-4} \cmidrule(r){5-8} \cmidrule(r){9-12} \cmidrule(r){13-16}
& & OV & \text{RT} & SQ & \text{RT} & {$R_{Opt}$} & {$T_{Opt}$} & OV & RT & {$R_{Opt}$} & {$T_{Opt}$} & OV & RT & {$R_{Opt}$} & {$T_{Opt}$} \\
\midrule
20 & 6630.29 & 6632.17 & 9.47 & 6632.17 & 0.39 & 0.00 & 95.71 & 6630.78 & 0.38 & 72.34 & 95.81 & 6630.33 & 0.53 & 97.32 & 94.12 \\
30 & 6624.47 & 6626.11 & 9.04 & 6626.16 & 0.39 & -2.20 & 95.57 & 6624.90 & 0.40 & 71.54 & 95.47 & 6624.52 & 0.60 & 96.95 & 93.03 \\
40 & 6624.61 & 6626.28 & 12.23 & 6626.29 & 0.39 & -0.79 & 96.63 & 6625.03 & 0.39 & 73.51 & 96.67 & 6624.66 & 0.65 & 96.80 & 94.43 \\
50 & 6632.66 & 6634.08 & 10.21 & 6634.05 & 0.40 & 1.04 & 96.02 & 6633.02 & 0.41 & 72.34 & 95.91 & 6632.71 & 0.61 & 96.06 & 93.77 \\
60 & 6625.43 & 6627.04 & 12.95 & 6627.00 & 0.42 & 2.39 & 96.62 & 6625.91 & 0.41 & 69.43 & 96.64 & 6625.48 & 0.70 & 96.82 & 94.30 \\
70 & 6620.91 & 6622.27 & 10.40 & 6622.31 & 0.38 & -1.46 & 96.00 & 6621.25 & 0.39 & 71.44 & 95.89 & 6620.96 & 0.62 & 96.20 & 93.38 \\
80 & 6616.00 & 6617.65 & 10.61 & 6617.65 & 0.37 & 0.00 & 96.31 & 6616.38 & 0.39 & 75.25 & 96.21 & 6616.05 & 0.59 & 96.82 & 94.15 \\
90 & 6615.10 & 6616.58 & 12.74 & 6616.54 & 0.39 & 2.85 & 96.63 & 6615.49 & 0.38 & 72.34 & 96.66 & 6615.14 & 0.65 & 96.66 & 94.40 \\
100 & 6609.30 & 6610.72 & 12.88 & 6610.68 & 0.39 & 2.82 & 96.69 & 6609.68 & 0.37 & 72.74 & 96.78 & 6609.35 & 0.69 & 96.40 & 94.10 \\
\bottomrule
\end{tabular}
\end{table*}

\begin{comment}
\begin{figure}[t]
\centering
\includegraphics[width=\linewidth]{Figure/NResult.png}
\caption{Trend of ROpt for our proposed schemes on Benchmark-All with $n = 50$ to $500$. ROpt increases with $n$ for all schemes.}
\label{fig:NRes}
\end{figure}

\begin{figure}[t]
\centering
\includegraphics[width=\linewidth]{Figure/NTime.png}
\caption{Trend of TOpt for our proposed schemes on Benchmark-All with $n = 50$ to $500$. TOpt increases with $n$ for all schemes.}
\label{fig:NTime}
\end{figure}
\end{comment}

% 首先根据FMP\_LS的运行结果分析匹配阶段的KM-M算法。根据ROpt列的结果，FMP\_LS解基本与MP\_LS一致，在合理的误差范围内。产生的偏差是由于匹配阶段的最优解不唯一，KM-M和KM算法得到的方案可能不一致，在分区阶段相同算法的处理后，最终的解也有一定区别。该结果符合预期，印证了KM-M算法的正确性。
Firstly, we analyze the KM-M algorithm in the match stage based on the results of FIMP-LS. According to the ``$R_{Opt}$'' column results, the solutions of FIMP-LS are essentially consistent with MP$_\text{LS}$, within a reasonable range of difference. The differences are due to the non-uniqueness of the optimal solutions in the match stage, so the matching schemes obtained by the KM-M and KM algorithms might differ. After processing by the same algorithm in the partition stage, the final solutions also present differences. These results are as expected and validate the correctness of the KM-M algorithm. 
% 然后，根据TOpt列的结果可以印证KM-M算法在运行时间上的优化效果。随着n的增大，FMP\_LS优化了MP\_LS百分之89.69到95.56的运行时间，即加速了10到20倍。原因是MP\_LS算法的时间瓶颈就在匹配阶段，而KM-M将匹配阶段的理论复杂度从$O(n^3)$降到了$O(n^2)$，因此能大幅优化算法运行效率，且n的规模越大，优化效果越明显。
Then, based on the results in the ``$T_{Opt}$'' column, it can be confirmed that the KM-M algorithm significantly improves the runtime efficiency. As $n$ increases, FIMP-LS optimizes the runtime of MP$_\text{LS}$ by 89.69\% to 95.56\%, effectively accelerating it by 10 to 20 times. The reason is that the time bottleneck of MP$_\text{LS}$ lies in the match phase, and the KM-M algorithm reduces the theoretical complexity of this phase from $O(n^3)$ to $O(n^2)$. Therefore, the improvement in runtime efficiency will increase as $n$ increases.

% 通过FMP\_ILS-GD的ROpt列的结果，可以分析three level local seach的优化效果。在只替换了patition phase的局部搜索下，运行结果随n的规模相较于MP\_LS逼近了最优解21.92\%到82.09\%，大幅度优化了解的质量。证明three level local seach可以通过多级的搜索策略更充分的对解空间进行搜索，并且在n越大即越复杂的场景下优化效果更明显。根据TOpt列，FMP\_ILS-GD在运行时间上相较FMP\_LS也有一些提升，这是由于能更快找到高质量的解，从而减少迭代轮数。
By analyzing the results in the ``$R_{Opt}$'' column of FIMP-MLS-GD, we can analyze the optimization effect of the MLS algorithm. When only the local search in the partition stage is replaced, the solution approaches the optimal solution by 21.92\% to 82.09\% as the $n$ increases, remarkably enhancing the solution quality compared to MP$_\text{LS}$. As $n$ increases, the complexity of the problem escalates, leading to more substantial optimization potential for solutions obtained by MP$_\text{LS}$. The result demonstrates that the MLS algorithm can explore the solution space more extensively. According to the ``$T_{Opt}$'' column, FIMP-MLS-GD shows some improvement in runtime compared to FIMP-LS. This improvement is due to its ability to find high-quality solutions more quickly, thereby reducing the number of iterations.

%FMP\_HEE-GD更进一步通过混合进化策略对partition phase进行优化，并细化二分图的修改规则。根绝ROpt列，运行结果随n的规模优化了73.18\%到88.04\%，优化效果明显优于FMP\_ILS-GD。在小规模的算例上高质量解的组合方式更少，在解空间上更难被搜索到，更强调算法的搜索范围和多样性。因此越小规模的算例，FMP\_HEE-GD的优化效果更明显。根据TOpt列，算法的运行时间有一定的提高，主要有两个原因。一个是HEE算法需要更多的迭代时间。另一个是由于优化了二分图更新策略以及partition phase更好的搜索策略，在迭代过程中改进解的次数更多，因此会增加迭代轮数。
FIMP-HGA-GD further optimizes the partition stage through the hybrid genetic algorithm with elite strategy and the modification rules of the bipartite graph. According to the ``$R_{Opt}$'' column, the results show an optimization of 73.18\% to 88.04\% as $n$ increases, which is a significantly better optimization effect compared to FIMP-MLS-GD. In small-scale instances, with limited combinations of high-quality solutions, finding them within the solution space becomes more challenging. This emphasizes the necessity for algorithms to have a wider search range and greater diversity. Therefore, the smaller the scale of the instances, the more effective the optimization effect of FIMP-HGA-GD. Based on the "$T_{Opt}$" column, we observe an increase in runtime, attributable primarily to two factors. One is that HGA requires more search time. The other is that the number of iterations increases due to more frequent improvements of the solution during the iterative process.

% 最后，通过比较FMP\_ILS-GD，FMP\_HEE-GD和FMP\_ILS，FMP\_HEE的运行结果，分析使用multi-way Karmarkar-karp算法生成初始解的效果。分析ROpt列可以得出，该初始解算法对解的优度带来了显著的提升。该方法能让初始分区方案更均衡，使得后续的局部搜索通过简单调整就能得到高质量的解。对于FMP\_ILS-GD算法，大幅提升了最终解的优度，n越小，提升越明显。对于FMP\_HEE-GD算法，将其在所有算例上的优化幅度都提升到了90\%以上。因此对于partition phase，高质量的初始解是很重要的，能使局部搜索阶段的探索更高效，更有价值。
Finally, by comparing the results of FIMP-MLS-GD, FIMP-MLS, FIMP-HGA-GD, and FIMP-HGA, we analyze the effectiveness of using the KK algorithm to generate initial solutions. Analysis of the ``$R_{Opt}$'' columns indicates that this initial solution algorithm significantly enhances the quality of final solutions. The KK algorithm generates more balanced initial partitioning schemes, enabling subsequent local searches to attain high-quality solutions with simple adjustments. Compared to FIMP-MLS-GD, FIMP-MLS significantly enhances the quality of the final solution, especially for instances with smaller $n$ values. For FIMP-HGA, it raises the optimization extent to over 90\% in all instances. Therefore, for the partition stage, high-quality initial solutions are crucial as they make the exploration in the local search phase more efficient and valuable.

\begin{table*}[t]
\caption{Improvement of all algorithms over MP$_\text{LS}$ on Benchmark-Con.}
\label{tbl:con}
\centering
\sisetup{
  mode=text, 
  detect-all
}
\begin{tabular}{@{}cSSSSSSSSSS@{}}
\toprule
\multirow{2}{*}{$Con$} & \multicolumn{2}{c}{FIMP-LS} & \multicolumn{2}{c}{FIMP-MLS-GD} & \multicolumn{2}{c}{FIMP-MLS} & \multicolumn{2}{c}{FIMP-HGA-GD} & \multicolumn{2}{c}{FIMP-HGA} \\
\cmidrule(r){2-3} \cmidrule(r){4-5} \cmidrule(r){6-7} \cmidrule(r){8-9} \cmidrule(r){10-11}
& {$R_{Opt}$} & {$T_{Opt}$} & {$R_{Opt}$} & {$T_{Opt}$} & {$R_{Opt}$} & {$T_{Opt}$} & {$R_{Opt}$} & {$T_{Opt}$} & {$R_{Opt}$} & {$T_{Opt}$} \\
\midrule
0  & 0.00 & 92.41 & 86.73 & 95.09 & 87.83 & 94.95 & 93.91 & 75.48 & 96.34 & 82.93 \\
10  & 0.00 & 95.08 & 72.08 & 95.34 & 77.30 & 95.33 & 88.81 & 92.97 & 97.30 & 94.32 \\
20  & 2.40 & 95.06 & 73.56 & 95.20 & 77.25 & 95.13 & 87.14 & 92.53 & 96.70 & 94.27 \\
30  & 0.00 & 95.51 & 64.02 & 95.71 & 67.49 & 95.52 & 82.70 & 93.14 & 95.68 & 94.07 \\
40  & 0.00 & 96.20 & 60.75 & 96.21 & 62.04 & 96.15 & 76.93 & 93.74 & 95.72 & 94.01 \\
50  & 0.00 & 95.67 & 63.05 & 95.58 & 68.83 & 95.57 & 80.31 & 92.20 & 96.97 & 93.22 \\
60  & 3.28 & 95.91 & 62.91 & 95.81 & 63.26 & 95.81 & 78.01 & 92.80 & 96.21 & 93.49 \\
70  & -4.32 & 95.89 & 64.35 & 95.84 & 70.07 & 95.79 & 76.09 & 92.23 & 96.36 & 93.09 \\
80  & -0.55 & 96.37 & 56.14 & 96.39 & 66.66 & 96.29 & 73.51 & 93.41 & 96.59 & 93.62 \\
90  & -5.14 & 96.72 & 49.97 & 96.81 & 53.48 & 96.65 & 68.86 & 93.63 & 95.18 & 93.50 \\
100  & 5.62 & 97.03 & 39.39 & 97.06 & 38.94 & 96.88 & 61.74 & 93.70 & 87.66 & 93.47 \\
\bottomrule
\end{tabular}
\end{table*}

\begin{table*}[t]
\caption{Improvement of four algorithms over MP$_\text{LS}$ on Benchmark-M and their concrete objective value.}
\label{tbl:m}
\centering
\sisetup{
  mode=text, 
  detect-all
}
\begin{tabular}{@{}ccS[table-format=5.2]SSSSSSSSS@{}}
\toprule
\multirow{2}{*}{$m$} & \multirow{2}{*}{\text{LB}} & \multicolumn{2}{c}{MP$_\text{LS}$} & \multicolumn{2}{c}{FIMP-MLS-GD} & \multicolumn{2}{c}{FIMP-MLS} & \multicolumn{2}{c}{FIMP-HGA-GD} & \multicolumn{2}{c}{FIMP-HGA} \\
\cmidrule(r){3-4} \cmidrule(r){5-6} \cmidrule(r){7-8} \cmidrule(r){9-10} \cmidrule(r){11-12}
& & \text{OV} & \text{Diff} & $V_{Opt}$ & $R_{Opt}$ & $V_{Opt}$ & $R_{Opt}$ & $V_{Opt}$ & $R_{Opt}$ & $V_{Opt}$ & $R_{Opt}$ \\
\midrule
2 & 65770.00 & 65770.10 & 0.10 & 0.00 & 0.00 & 0.10 & 100.00 & 0.10 & 100.00 & 0.10 & 100.00 \\
7 & 18919.94 & 18920.16 & 0.22 & 0.10 & 31.67 & 0.16 & 53.33 & 0.17 & 71.11 & 0.22 & 100.00 \\
12 & 11042.52 & 11043.04 & 0.52 & 0.21 & 36.06 & 0.41 & 76.48 & 0.34 & 63.68 & 0.52 & 100.00 \\
17 & ~7804.26 & 7805.18 & 0.92 & 0.31 & 28.94 & 0.53 & 55.65 & 0.56 & 58.08 & 0.90 & 97.17 \\
22 & ~6026.37 & 6028.09 & 1.72 & 0.99 & 57.25 & 1.15 & 66.57 & 1.23 & 71.33 & 1.62 & 94.52 \\
27 & ~4903.17 & 4906.37 & 3.20 & 2.35 & 73.33 & 2.40 & 74.46 & 2.61 & 81.18 & 2.95 & 92.02 \\
32 & ~4132.20 & 4136.70 & 4.49 & 3.52 & 77.98 & 3.61 & 79.56 & 3.75 & 82.98 & 4.09 & 90.92 \\
\bottomrule
\end{tabular} 
\end{table*}

\begin{figure}[t]
\centering
\includegraphics[width=\linewidth]{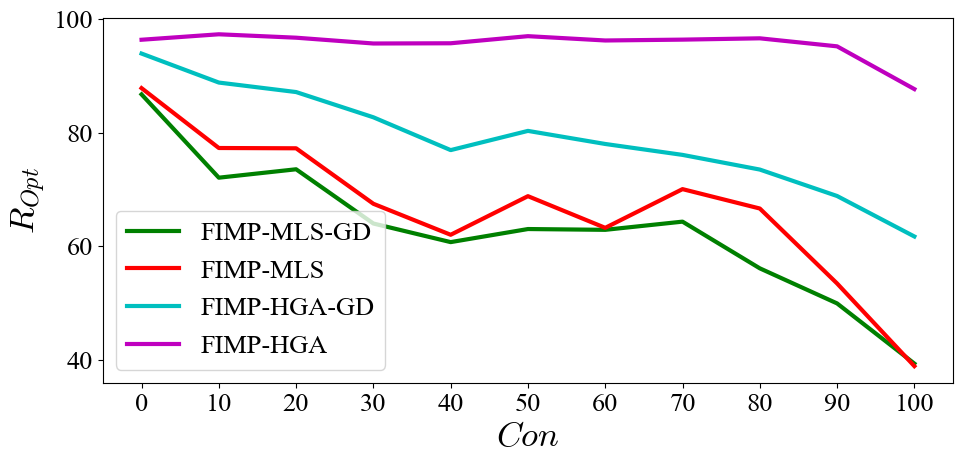}
\caption{Trend of $R_{Opt}$ for our proposed algorithms on Benchmark-Con with $Con = 0$ to $100$. Except for FIMP-HGA, the $R_{Opt}$ of other schemes obviously decreases as $Con$ increases.}
\label{fig:ConRes}
\end{figure}

\begin{figure}[t]
\centering
\begin{subfigure}{0.24\textwidth}
  \includegraphics[width=\linewidth]{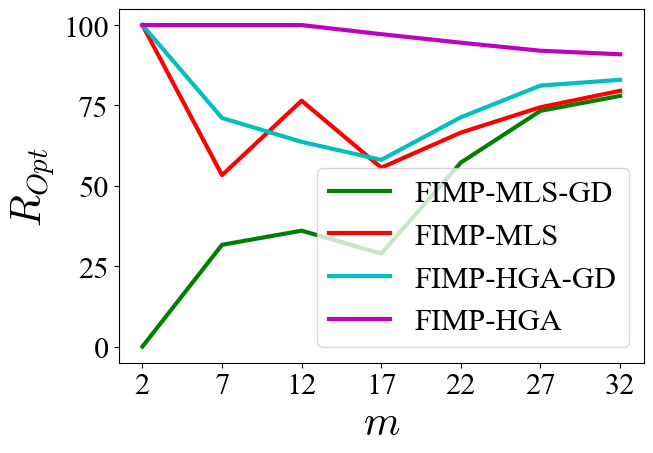}
  \caption{}
  \label{fig:MROpt}
\end{subfigure}\hfil
\begin{subfigure}{0.24\textwidth}
  \includegraphics[width=\linewidth]{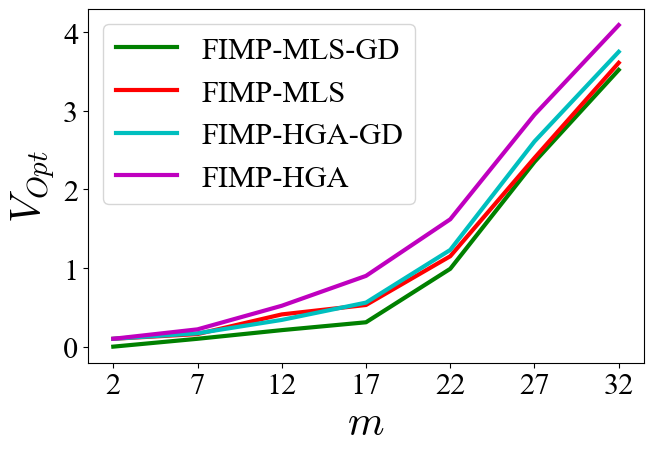}
  \caption{}
  \label{fig:MVOpt}
\end{subfigure}
\caption{Trend of $R_{Opt}$ (a) and $V_{Opt}$ (b) for our proposed algorithms on Benchmark-M with $m = 2$ to $32$. $R_{Opt}$ does not show a clear pattern with changes in $m$. $V_{Opt}$ increases as $m$ increases for all algorithms.}
\label{fig:M}
\end{figure}

\subsubsection{Evaluation based on Benchmark Features}

% 在\rtbl{den}中展示了MP\_LS，FMP\_LS,FMP\_ILS和FMP\_HEE的运行结果随二分图密度的变化情况，并在Expect列中列出了理论下界的平均值，Result列和Time列展示了各个方案实际解的评估函数值和运行时间。随着二分图密度的增加，match stage的可行方案数增加，因此能得到更优的解，但也需要更多的时间搜索解空间。可以看出虽然MP\_LS的运行结果已经很接近理论下界，但我们提出的方案。
% 对MPLS的OV和RT列分析，随着二分图密度的增加，objective value和runtime分别呈减小和增加趋势。这是因为随着边数的增加，match stage能得到总权重更小的匹配方案，且解空间中有更多可探索的解，但也需要更多的搜索轮次和时间
In \rtbl{den}, we present the performance of MP$_\text{LS}$, FIMP-LS, FIMP-MLS, and FIMP-HGA as the density of the bipartite graph changes. The average theoretical lower bound, objective value, and runtime are respectively listed in the ``LB'', ``OV'', and ``RT'' columns. 
Analysis of the ``OV'' and ``RT'' columns for MP$_\text{LS}$ shows that as the density of the bipartite graph increases, the objective value decreases while the runtime increases. This trend is due to the match stage achieving match schemes with lower total weights as the number of edges increases. Consequently, there are more solutions to explore in the solution space, necessitating additional search rounds and more runtimes.
% As the density of the bipartite graph increases, the number of perfect matches grows, leading to better outcomes but also requiring more time to explore the solution space. 
% It is observed that the results of MP$_\text{LS}$ are already close to the theoretical lower bounds.
% Based on the Result and Time columns, the actual evaluation function value of the solution and running times of each scheme are shown. Although the results of MP$_\text{LS}$ are already close to the theoretical lower bound, FIMP-HGA can still achieve substantial improvements. From the analysis, it can be concluded that as the density of the bipartite graph increases, each vertex has more edges to choose from, resulting in better solutions, but also requiring more time to try different schemes to enhance the quality of the solution.
% 相较于MPLS，我们提出的算法主要在partition stage提升解的优度。
% 根据ROpt列，FIMP-MLLS和FIMP-HGA对MP\_LS的改进幅度基本不随二分图密度变化。这是因为我们提出的算法主要在partition stage提升解的优度。
According to the $R_{Opt}$ column, the improvement margin of FIMP-MLS and FIMP-HGA over MP$_\text{LS}$ remains consistent across various bipartite graph densities. This is because our proposed algorithms primarily enhance the solution quality during the partition stage. Although the outcomes of MP$_\text{LS}$ are already close to the theoretical lower bound, our proposed FIMP-HGA can still achieve at least a 96\% improvement.
% According to the ROpt column, the extent of improvement of different schemes over MP$_\text{LS}$ is essentially consistent across various bipartite graph densities.

% 在\rtbl{clu}中展示了各个方案的运行表现随二分图一致性的变化情况。\rfig{CluRes}整理了FMP\_ILS-GD、FMP\_ILS、FMP\_HEE-GD和FMP\_HEE对MP\_LS的运行结果。随着二分图一致性的增加FMP\_ILS-GD、FMP\_ILS和FMP\_HEE-GD的优化幅度呈降低趋势。即二分图的一致性越大，对MP\_LS的解改进的难度越大，离理论下界越远。这印证了\rsec{con}中的对一致性的描述。但对于FMP\_HEE在二分图不同的一致性下均能保证对MP\_LS大幅改进，得到接近理论下界的解。证明FMP\_HEE能准确的识别高质量的解空间，并进行充分的搜索，对PMMWM的求解效果有本质上的改进。
In \rtbl{con}, we demonstrate how the performance of each algorithm varies with the changes in the consistency of the bipartite graph. \rfig{ConRes} organizes the improvement ratio of FIMP-LS-GD, FIMP-MLS, FIMP-HGA-GD, and FIMP-HGA compared to MP$_\text{LS}$. As the consistency of the bipartite graph increases, the improvement ratio of FIMP-LS-GD, FIMP-MLS, and FIMP-HGA-GD shows a decreasing trend. That is, the greater the consistency of the bipartite graph, the more difficult it is to improve upon solutions of MP$_\text{LS}$, and the further away from the theoretical lower bound it becomes. This confirms the description of consistency in \rsec{DS}. However, FIMP-HGA consistently achieves a significant improvement ratio across different consistency of bipartite graphs, obtaining solutions close to the theoretical lower bound. This proves that FIMP-HGA can accurately identify high-quality solution spaces and conduct thorough searches, fundamentally enhancing the solution quality for PMMWM.

In \rtbl{m}, we present the performance of FIMP-MLS-GD, FIMP-MLS, FIMP-HGA-GD, and FIMP-HGA as the number of partitions varies. For MP$_\text{LS}$ the ``Diff'' column represents the difference between the calculated objective function values (``OV'' column) and the theoretical lower bounds (``LB'' column).
% In each scheme, the VOpt column indicates the actual value by which the results are optimized compared to MP\_LS. These are the average outcomes of multiple data points run multiple times.
% 根据Expect和Diff列的结果，随着分区数量的增加，MP\_LS实际运行结果与理论下界差距越大，有更大的改进空间。这是因为分区数量变多会导致将权重均匀的分配到各个partition的难度增加。直接分析优化的实际值，即\rfig{MVOpt}中为VOpt列的结果，可得随着分区数量的增加，各个方案优化的实际值均增加。若分析各个方案的改进幅度，即ROpt列，如\rfig{MROpt}所示，不是单调增加的。这是因为虽然优化效果更好，但随分区数的增加，理论下界会变得更难达到，受这两个因素共同影响改进。
% 直接分析各个方案的改进幅度，即ROpt列，如\rfig{MROpt}所示。由于优化幅度不仅受优化效果影响，同时随分区数的增加，理论下界会变得更难达到，受这两个因素影响改进幅度不是单调变化的。而直接分析优化的实际值，即\rfig{MVOpt}中为VOpt列的结果，可得随着分区数量的增加，各个方案优化的实际值均增加。
According to the results in the "Diff" columns, as the number of partitions increases, the difference becomes larger, indicating greater room for improvement. This is due to the increased difficulty in evenly distributing weights across more partitions. 
% By directly analyzing the actual value of the optimization, as shown in \rfig{MVOpt} for the VOpt column results, the optimization for each scheme increases with the number of partitions. However, when analyzing the improvement extent of each scheme, i.e., the ROpt column, as illustrated in \rfig{MROpt}, the increase is not monotonic. This is because the theoretical lower bound, a reference value for evaluating solutions that is no greater than the optimal evaluation function value, becomes more difficult to reach with an increasing number of partitions.
By directly analyzing the ``$R_{Opt}$'' column, as shown in \rfig{MROpt}, we can see that the improvement ratio over MP$_\text{LS}$ is not a monotonic change. This is because $R_{Opt}$ is influenced not only by the optimization for the solution quality but also by the difficulty of reaching the theoretical lower bound, which increases with the number of partitions. By directly analyzing the actual improvement of the solution quality, namely the values in the $V_{Opt}$ column shown in \rfig{MVOpt}, it can be observed that the improvement in the objective value of each scheme over MP$_\text{LS}$ increases as the number of partitions increases.

\section{Conclusion}
\label{subsec:con}
In this study, we introduce FIMP-HGA as a solution to the PMMWM problem, employing the Match-Partition decomposition framework akin to MP$_\text{LS}$. Notable advancements are included in FIMP-HGA. For example, we introduce an exact algorithm during the match stage via incremental updates, which reduces time complexity per iteration from $O(n^3)$ to $O(n^2)$. The partition stage in FIMP-HGA features a Hybrid Genetic Algorithm with an elite strategy, complemented by our proposed Greedy Partition Crossover operator and Multilevel Local Search, significantly enhancing solution quality. Dynamic adjustments to the bipartite graph structure ensure a thorough exploration of the solution space. Our experiments, conducted on four proposed benchmarks encompassing 1750 instances, demonstrate the superiority of FIMP-HGA over the state-of-the-art method. FIMP-HGA achieves results closer to the theoretical optimum by at least 90\% compared to MP$_\text{LS}$ while reducing the computational time by 3-20 times. 
%Additionally, ablation tests on FIMP-HGA validate the optimization effects of the proposed modules. We also highlight the impact of different instance features on algorithm performance.

%In this paper, we significantly improve the PMMWM algorithm through exact and heuristic optimizations. 
In future work, we plan to refine and optimize the hybrid genetic algorithm for the partition stage and explore its application to related problems such as Multiprocessor Scheduling and Multi-way Number Partition. Additionally, we aim to enhance the interaction between match and partition stages by further analyzing and experimenting with bipartite graph structure adjustments.

\begin{comment}
\section*{Acknowledgments}
This should be a simple paragraph before the References to thank those individuals and institutions who have supported your work on this article.

{\appendix[Proof of the Zonklar Equations]
Use $\backslash${\tt{appendix}} if you have a single appendix:
Do not use $\backslash${\tt{section}} anymore after $\backslash${\tt{appendix}}, only $\backslash${\tt{section*}}.
If you have multiple appendixes use $\backslash${\tt{appendices}} then use $\backslash${\tt{section}} to start each appendix.
You must declare a $\backslash${\tt{section}} before using any $\backslash${\tt{subsection}} or using $\backslash${\tt{label}} ($\backslash${\tt{appendices}} by itself
 starts a section numbered zero.)}
\end{comment}

%{\appendices
%\section*{Proof of the First Zonklar Equation}
%Appendix one text goes here.
% You can choose not to have a title for an appendix if you want by leaving the argument blank
%\section*{Proof of the Second Zonklar Equation}
%Appendix two text goes here.}

\bibliographystyle{IEEEtranN}
\bibliography{00-Bib}

\begin{IEEEbiography}
[{\includegraphics[width=1in,height=1.25in,clip,keepaspectratio]{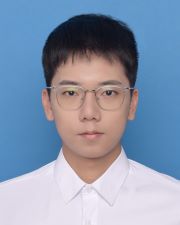}}]{Yuxuan Wang} 
is currently pursuing the Ph.D. degree with Huazhong University of Science and Technology, Wuhan, China. He received the B.S. degree in School of Computer Science and Technology from Huazhong University of Science and Technology, Wuhan, China, in 2022. His research interests include combinatorial optimization and artificial intelligence.
\end{IEEEbiography}

\begin{IEEEbiography}
[{\includegraphics[width=1in,height=1.25in,clip,keepaspectratio]{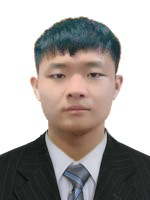}}]{Jiongzhi Zheng}
is currently pursuing the Ph.D. degree with Huazhong University of Science and Technology, Wuhan, China. He received the B.S. degree in material forming and its control engineering from Huazhong University of Science and Technology, Wuhan, China, in 2018. His research interests include combinatorial optimization and reinforcement learning.
\end{IEEEbiography}

\begin{IEEEbiography}
[{\includegraphics[width=1in,height=1.25in,clip,keepaspectratio]{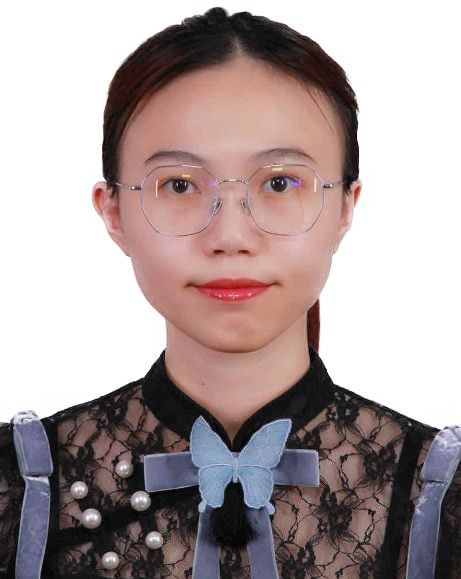}}]{Jinyao Xie} is currently working at Huawei Research Institute, Wuhan, China. She received the B.S. degree in school of computer science and technology from Huazhong University of Science and Technology, Wuhan, China, in 2021.

\end{IEEEbiography}

\begin{IEEEbiography}
[{\includegraphics[width=1in,height=1.25in,clip,keepaspectratio]{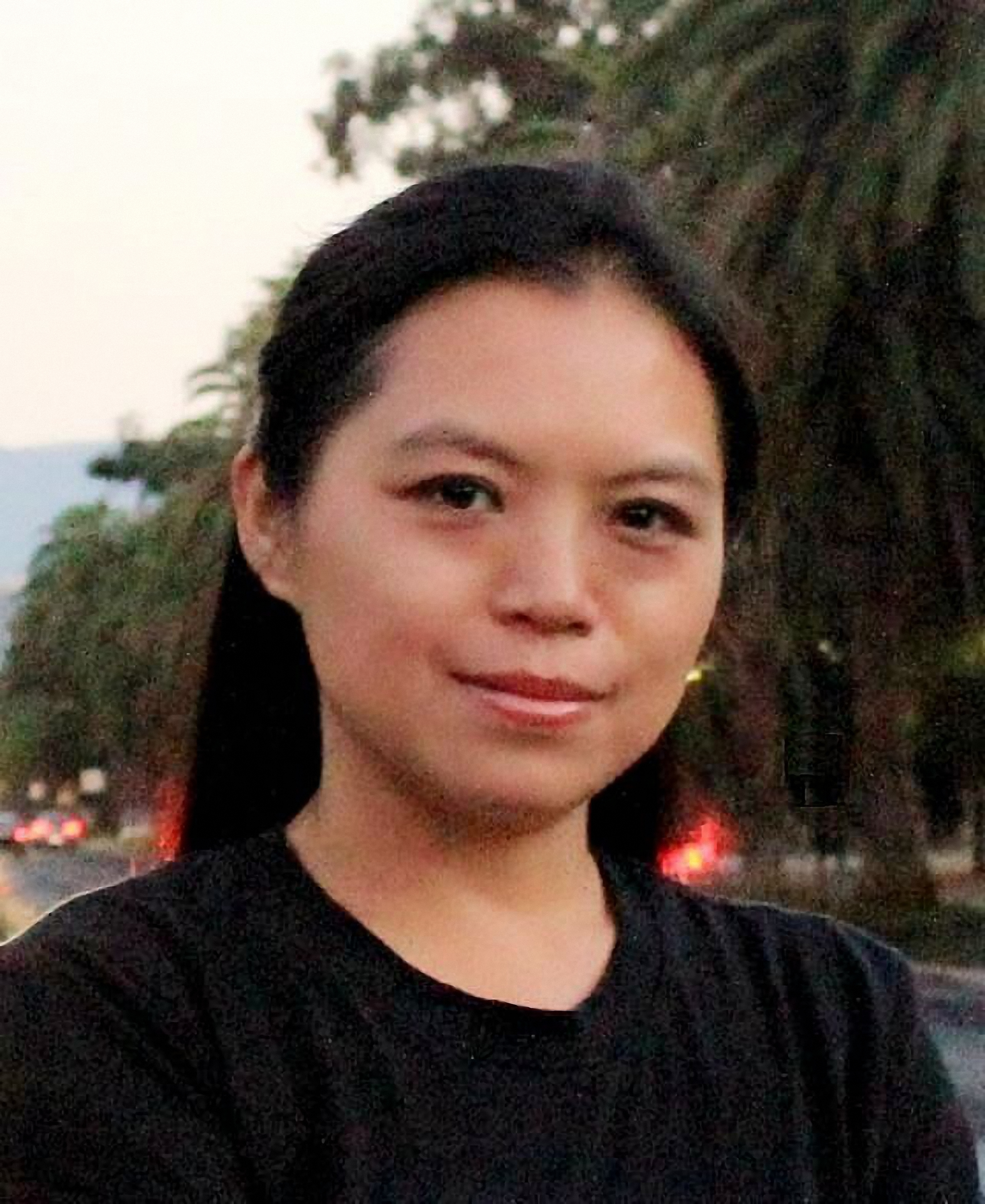}}]{Kun He}
(SM18) is currently a Professor in School of Computer Science and Technology, Huazhong University of Science and Technology, Wuhan, P.R. China. She received the Ph.D. degree in system engineering from Huazhong University of Science and Technology, Wuhan, China, in 2006. She had been with the Department of Management Science and Engineering at Stanford University in 2011-2012 as a visiting researcher. She had been with the department of Computer Science at Cornell University in 2013-2015 as a visiting associate professor, in 
2016 as a visiting professor, and in 2018 as a visiting professor. She was honored as a Mary Shepard B. Upson visiting professor for the 2016-2017 Academic year in Engineering, Cornell University, New York. Her research interests include adversarial machine learning, representation learning, graph data mining, and combinatorial optimization. 
\end{IEEEbiography}

\iffalse 
\section{Biography Section}
If you have an EPS/PDF photo (graphicx package needed), extra braces are
 needed around the contents of the optional argument to biography to prevent
 the LaTeX parser from getting confused when it sees the complicated
 $\backslash${\tt{includegraphics}} command within an optional argument. (You can create
 your own custom macro containing the $\backslash${\tt{includegraphics}} command to make things
 simpler here.)
 \fi 

\begin{comment}
 
\vspace{11pt}

\bf{If you include a photo:}\vspace{-33pt}
%\begin{IEEEbiography}[{\includegraphics[width=1in,height=1.25in,clip,keepaspectratio]{fig1}}]{Michael Shell}
Use $\backslash${\tt{begin\{IEEEbiography\}}} and then for the 1st argument use $\backslash${\tt{includegraphics}} to declare and link the author photo.
Use the author name as the 3rd argument followed by the biography text.
%\end{IEEEbiography}

\vspace{11pt}

\bf{If you will not include a photo:}\vspace{-33pt}
\begin{IEEEbiographynophoto}{John Doe}
Use $\backslash${\tt{begin\{IEEEbiographynophoto\}}} and the author name as the argument followed by the biography text.
\end{IEEEbiographynophoto}

\end{comment}

\vfill

\end{document}